%% file: asgnht_aistats.tex
\documentclass[twoside]{article}
\usepackage[accepted]{aistats2016}
\usepackage{natbib}
\bibpunct{(}{)}{;}{a}{,}{,}
\input{tex/preamble}

\begin{document}

%
\runningtitle{Bridging the Gap between SG-MCMC and Stochastic Optimization}

%
\runningauthor{C. Chen, D. Carlson, Z. Gan, C. Li, and L. Carin}

\twocolumn[

\aistatstitle{Bridging the Gap between Stochastic Gradient MCMC\\ and Stochastic Optimization}

\aistatsauthor{\hspace{6mm} Changyou Chen$^\dag$ \And \hspace{4mm}David Carlson$^\ddag$ \And Zhe Gan$^\dag$ \And Chunyuan Li$^\dag$ \And Lawrence Carin$^\dag$ }
\aistatsaddress{ $^\dag$Department of Electrical and Computer Engineering, Duke University}
\vspace{-0.7cm}\aistatsaddress{$^\ddag$Department of Statistics and Grossman Center for Statistics of Mind, Columbia University } ]

\input{tex/abstract}
\input{tex/introduction}

\input{tex/preliminaries}

\input{tex/algorithm}

\input{tex/theory_dec}
\input{tex/experiments}

\input{tex/conclusions}

\newpage
\input{tex/ack}
\bibliographystyle{abbrvnat}
\bibliography{asgnht_aistats}

\newpage
\twocolumn[

\aistatstitle{Bridging the Gap between Stochastic Gradient MCMC and Stochastic Optimization: Supplementary Material}

\aistatsauthor{ Changyou Chen$^\dag$ \And David Carlson$^\ddag$ \And Zhe Gan$^\dag$ \And Chunyuan Li$^\dag$ \And Lawrence Carin$^\dag$ }
\aistatsaddress{ $^\dag$Department of Electrical and Computer Engineering, Duke University}
\vspace{-0.7cm}\aistatsaddress{$^\ddag$Department of Statistics and Grossman Center for Statistics of Mind, Columbia University} ]

\appendix
\input{tex/appendix}
\end{document}

%% file: tex/preamble.tex
\usepackage{hyperref}
\usepackage{url}
\usepackage{amsmath}
\usepackage{amsfonts}

\usepackage{amssymb}
\usepackage{amsthm}

\usepackage{tikz}
\usepackage{adjustbox}

\usepackage{graphicx} 
\usepackage{subfigure}
\usepackage{floatrow}
\graphicspath{{figures/}}

\usepackage[lined,boxed,commentsnumbered,ruled]{algorithm2e}

\input{Definitions}

\ifx\proof\undefined
\newenvironment{proof}{\par\noindent{\bf Proof\ }}{\hfill\BlackBox\\[2mm]}
\fi

\ifx\theorem\undefined
\newtheorem{theorem}{Theorem}
\fi
\ifx\example\undefined
\newtheorem{example}{Example}
\fi
\ifx\lemma\undefined
\newtheorem{lemma}[theorem]{Lemma}
\fi

\ifx\corollary\undefined
\newtheorem{corollary}[theorem]{Corollary}
\fi

\ifx\assumption\undefined
\newtheorem{assumption}{Assumption}
\fi

\ifx\definition\undefined
\newtheorem{definition}{Definition}
\fi

\ifx\proposition\undefined
\newtheorem{proposition}[theorem]{Proposition}
\fi

\ifx\remark\undefined
\newtheorem{remark}[theorem]{Remark}
\fi

\ifx\conjecture\undefined
\newtheorem{conjecture}[theorem]{Conjecture}
\fi

\ifx\factoid\undefined
\newtheorem{factoid}[theorem]{Fact}
\fi

\ifx\axiom\undefined
\newtheorem{axiom}[theorem]{Axiom}
\fi

\newcommand{\RN}[1]{%
  \textup{\lowercase\expandafter{\it \romannumeral#1}}%
}

%% file: tex/abstract.tex
\begin{abstract}
Stochastic gradient Markov chain Monte Carlo (SG-MCMC) methods are Bayesian analogs to popular stochastic optimization methods; however, this connection is not well studied.
We explore this relationship by applying simulated annealing to an SG-MCMC algorithm.   
Furthermore, we extend recent SG-MCMC methods with two key components: $\RN{1})$ adaptive preconditioners (as in ADAgrad or RMSprop), and $\RN{2})$ adaptive element-wise momentum weights.  
The zero-temperature limit gives a novel stochastic optimization method with \emph{adaptive element-wise} momentum weights, while conventional optimization methods only have a shared, static momentum weight.  
Under certain assumptions, our theoretical analysis suggests the proposed simulated annealing approach converges close to the global optima.  
Experiments on several deep neural network models show state-of-the-art results compared to related stochastic optimization algorithms.

\end{abstract}

%% file: tex/introduction.tex
\section{Introduction}

Machine learning has made significant recent strides due to large-scale learning applied to ``big data''. Large-scale learning is typically performed with stochastic optimization, and the most common method is stochastic gradient descent (SGD) \citep{bottou2010}.  Stochastic optimization methods are devoted to obtaining a (local) optima of an objective function.  Alternatively, Bayesian methods aim to compute the expectation of a test function over the posterior distribution.  At first glance, these methods appear to be distinct, independent approaches to learning.  However, even the celebrated Gibbs sampler was first introduced to statistics as a simulated annealing method for \textit{maximum a posteriori} estimation (\ie finding an optima) \citep{geman1984}.

Recent work on large-scale Bayesian learning has  focused on incorporating the speed and low-memory costs from stochastic optimization.  These approaches are referred to as stochastic gradient Markov chain Monte Carlo (SG-MCMC) methods.  Well-known SG-MCMC methods include stochastic gradient Langevin dynamics (SGLD) \citep{WellingT:ICML11}, 
stochastic gradient Hamiltonian Monte Carlo (SGHMC) \citep{ChenFG:ICML14}, and stochastic gradient 
thermostats (SGNHT) \citep{DingFBCSN:NIPS14}. SG-MCMC has become increasingly popular in the literature 
due to practical successes, ease of implementation,
and theoretical convergence properties \citep{TehTV:arxiv14,VollmerZT:arxiv15,ChenDC:NIPS15}. 

There are obvious structural similarities between SG-MCMC algorithms and stochastic optimization methods.  For example, SGLD resembles SGD with additive Gaussian noise.  SGHMC resembles SGD with momentum \citep{RumelhartHW:nature86}, adding additive Gaussian noise when updating the momentum terms \citep{ChenFG:ICML14}.  
These similarities are detailed in Section \ref{sec:preliminaries}.
Despite these structural similarities, the theory is unclear on how additive Gaussian noise differentiates a Bayesian algorithm from its optimization analog.

Just as classical sampling methods were originally used for optimization \citep{geman1984}, we directly address using SG-MCMC algorithms for optimization.  A major benefit of adapting these schemes is that Bayesian learning is (in theory) able to fully explore the parameter space.  Thus it may find a better local optima, if not the global optima, for a non-convex objective function.

Specifically, in this work we first extend the recently proposed multivariate stochastic gradient thermostat algorithm \citep{GanCHCC:icml15} with Riemannian information geometry, which results in an adaptive preconditioning and momentum scheme with analogs to Adam \citep{kingma2014adam} and RMSprop \citep{tieleman2012lecture}.  We propose an annealing scheme on the system temperature to move from a Bayesian method to a stochastic optimization method.  We call the proposed algorithm Stochastic AnNealing Thermostats with Adaptive momentum (Santa).  We show that in the temperature limit, Santa recovers the SGD with momentum algorithm except that: $\RN{1}$) adaptive preconditioners are used when updating both model and momentum parameters; $\RN{2}$) each parameter has an individual, \emph{learned} momentum parameter.  Adaptive preconditioners and momentums are desirable in practice because of their ability to deal with uneven, dynamic curvature \citep{dauphin2015rmsprop}.  For completeness, we first review related algorithms in Section \ref{sec:preliminaries}, and present our novel algorithm in Section \ref{sec:alg}.

We develop theory to analyze convergence properties of our algorithm, suggesting that Santa is able to find a solution for an (non-convex) objective function close to its global optima, shown in Section \ref{sec:theory}.  The theory is based on the analysis from stochastic differential equations \citep{TehTV:arxiv14,ChenDC:NIPS15}, and presents results on bias and variance of the annealed Markov chain.  This is a fundamentally different approach from the traditional convergence explored in stochastic optimization, or the regret bounds used in online optimization.  We note we can adapt the regret bound of Adam \citep{kingma2014adam} for our zero-temperature algorithm (with a few trivial modifications) for a convex problem, as shown in Supplementary Section~\ref{sec:relate_adam}.  However, this neither addresses non-convexity nor the annealing scheme that our analysis does.  

In addition to theory, we demonstrate effective empirical performance on a variety of deep neural networks (DNNs), achieving the best performance compared to all competing algorithms for the same model size.  This is shown in Section \ref{sec:experiments}. 
The code is publicly available at \href{https://github.com/cchangyou/Santa}{https://github.com/cchangyou/Santa}.

%% file: tex/preliminaries.tex
\vspace{-2mm}
\section{Preliminaries}
\vspace{-2mm}
\label{sec:preliminaries}
Throughout this paper, we denote vectors as bold, lower-case letters, and matrices as bold, upper-case letters.  
We use $\odot$ for element-wise multiplication, and $\oslash$ as element-wise division;
$\sqrt{\cdot}$ denotes the element-wise square root when applied to vectors or matricies.
We reserve $(\cdot)^{1/2}$ for the standard matrix square root.
$\textbf{I}_p$ is the $p\times p$ identity matrix, $\textbf{1}$ is an all-ones vector.

The goal of an optimization algorithm is to minimize an objective function $U(\thetab)$ that corresponds to a (non-convex) model of interest.
In a Bayesian model, this corresponds to the potential energy defined as the negative log-posterior, $U(\thetab) \triangleq -\log p(\thetab) - \sum_{n = 1}^N \log p(\xb_n | \thetab)$.  Here $\thetab \in \RR^p$ are the model parameters, and $\{\xb_n\}_{n=1,\dots,N}$ are the $d$-dimensional observed data; $p(\thetab)$ corresponds to the prior and $p(\xb_n|\thetab)$ is a likelihood term for the $n^{th}$ observation.  In optimization, $-\sum_{n = 1}^N \log p(\xb_n | \thetab)$ is typically referred to as the loss function, and $-\log p(\thetab)$ as a regularizer.

In large-scale learning, $N$ is prohibitively large.  This motivates the use of stochastic approximations. We denote the stochastic approximation $\tilde{U}_t(\thetab) \triangleq -\log p(\thetab) - \frac{N}{m}\sum_{j = 1}^m \log p(\xb_{i_j} | \thetab)$,
where $(i_1, \cdots, i_m)$ is a random subset of the set $\{1, 2, \cdots, N\}$.  The gradient on this minibatch is denoted as $\tilde{\fb}_t(\thetab)=\nabla \tilde{U}_t(\thetab)$, which is an unbiased estimate of the true gradient.

A standard approach to learning is SGD, where parameter updates are given by $\thetab_{t}=\thetab_{t-1}-\eta_t \tilde{\fb}_{t-1}(\thetab)$ with $\eta_t$ the learning rate. This is guaranteed to converge to a local minima under mild conditions \citep{bottou2010}. The SG-MCMC analog to this is SGLD, with updates $\thetab_{t}=\thetab_{t-1}-\eta_t \tilde{\fb}_{t-1}(\thetab)+\sqrt{2\eta_t}\zetab_t$. The additional term is 
a standard normal random vector, $\zetab_t\sim \mathcal{N}({\bf 0}, \textbf{I}_p)$ \citep{WellingT:ICML11}.  The SGLD method draws approximate posterior samples instead of obtaining a local minima.

Using momentum in stochastic optimization is important in learning deep models \citep{SutskeverMDH:icml13}.  This motivates SG-MCMC algorithms with momentum.  
The standard SGD with momentum (SGD-M) approach introduces an auxiliary
variable $\u_t \in \RR^p$ to represent the momentum.  
Given a momentum weight $\alpha$, the updates are $\thetab_{t}=\thetab_{t-1}+\eta_t \u_t$ and $\u_{t}=(1-\alpha) \u_{t-1}-\tilde{\fb}_{t-1}(\thetab)$.  
A Bayesian analog is SGHMC \citep{ChenFG:ICML14} or multivariate SGNHT (mSGNHT) \citep{GanCHCC:icml15}.  
In mSGNHT, each parameter has a unique momentum weight $\alphab_t \in \RR^{p}$ that is learned during the sampling sequence.  
The momentum weights are updated to maintain the system temperature $1/\beta$.
An inverse temperature of $\beta=1$ corresponds to the posterior.  This algorithm has updates $\thetab_{t}=\thetab_{t-1}+\eta_t \u_t$, $\u_{t}=(\textbf{1}-\eta_t\alphab_{t-1}) \odot \u_{t-1}-\eta_t\tilde{\fb}_{t-1}(\thetab)+\sqrt{2 \eta_t/\beta}\zetab_t$.  
The main difference is the additive Gaussian noise and step-size dependent momentum update.  The weights have updates $\alphab_t=\alphab_{t-1}+\eta_t((\u_t \odot \u_t)- \bf{1} / \beta)$, which matches the kinetic energy to the system temperature.

A recent idea in stochastic optimization is to use an adaptive preconditioner, also known as a variable metric, to improve
convergence rates.  
Both ADAgrad \citep{duchi2011adaptive} and Adam \citep{kingma2014adam} adapt to the local geometry with a regret bound of $\mathcal{O}(\sqrt{N})$.  Adam adds momentum as well through moment smoothing.  
RMSprop \citep{tieleman2012lecture}, Adadelta \citep{zeiler2012adadelta}, and RMSspectral \citep{carlson2015preconditioned} are similar methods with preconditioners.
Our method introduces adaptive momentum and preconditioners to the SG-MCMC.  
This differs from stochastic optimization in implementation and theory, and is novel in SG-MCMC.

Simulated annealing \citep{Kirkpatrick:science83,Cerny:JOTA85} is well-established as a way of acquiring a local mode by moving from a high-temperature, flat surface to a low-temperature, peaky surface.
It has been explored in the context of MCMC, including reversible jump MCMC \citep{AndrieuFD:UAI00}, annealed important sampling \citep{Neal:SC01} and parallel tempering \citep{LiPAZG:AMC09}.
Traditional algorithms are based on Metropolis--Hastings sampling, which require computationally expensive accept-reject steps.
Recent work has applied simulated annealing to large-scale learning through mini-batch based annealing \citep{VandeMeentPW:arx14,ObermeyerGJ:AISTATS14}.
Our approach incorporates annealing into SG-MCMC with its inherent speed and mini-batch nature.

%% file: tex/algorithm.tex
\section{The Santa Algorithm}
\label{sec:alg}

\begin{algorithm}[t]
\SetKwInOut{Input}{Input}
\caption{Santa with the Euler scheme}
\Input{$\eta_t$ (learning rate), $\sigma$, $\lambda$, $burnin$,
$\beta = \{\beta_1, \beta_2, \cdots\}\rightarrow \infty$,  $\{\zetab_t \in \RR^p\}\sim \mathcal{N}({\bf 0},\textbf{I}_p)$.}
Initialize $\thetab_{0}$, $\ub_{0} = \sqrt{\eta}\times\mathcal{N}(0, I)$, $\alphab_0 = \sqrt{\eta}C$, $\vb_0 = 0$ \;
\For {$t = 1, 2, \ldots $} {
Evaluate $\tilde{\fb}_t \triangleq \nabla_{\thetab} \tilde{U}(\thetab_{t-1})$ on the $t^\text{th}$ mini-batch\;
$\vb_t = \sigma \vb_{t-1} + \frac{1 - \sigma}{N^2}\tilde{\fb}_t \odot \tilde{\fb}_t$ \;
$\gb_t = 1\oslash\sqrt{\lambda + \sqrt{\vb_t}}$ \;
\uIf{$t < burnin$}{
\tcc{{\em exploration}}
$\alphab_{t} = \alphab_{t-1} + \left(\ub_{t-1}\odot \ub_{t-1} - \eta/\beta_t\right)$\;
$\ub_{t} = \frac{\eta}{\beta_t}\left(1 - \gb_{t-1}\oslash \gb_{t}\right)\oslash \ub_{t-1} + \sqrt{\frac{2\eta}{\beta_t}\gb_{t-1}}\odot \zetab_t$
}
\Else{
\tcc{{\em refinement}}
$\alphab_{t} = \alphab_{t-1}$;
~~~$\u_t=\textbf{0}$\;
}
$\ub_{t}=\ub_{t}+\left(1 - \alphab_{t}\right)\odot \ub_{t-1} - \eta \gb_t\odot \tilde{\fb}_t$\;
$\thetab_{t} = \thetab_{t-1} + \gb_t \odot \ub_{t}$\; 
}
\label{alg:sahmc}
\end{algorithm}

 

Santa extends the mSGNHT algorithm with preconditioners and a simulated annealing scheme.
A simple pseudocode is shown in Algorithm \ref{alg:sahmc}, or a more complex, but higher accuracy version, is shown in Algorithm~\ref{alg:sahmc_ssi},
and we detail the steps below.

The first extension we consider is the use of adaptive preconditioners.  Preconditioning has been proven critical for fast convergence in both stochastic optimization \citep{dauphin2015rmsprop} and SG-MCMC algorithms \citep{PattersonT:NIPS13}.  
In the MCMC literature, preconditioning is alternatively referred to as \textit{Riemannian information geometry} \citep{PattersonT:NIPS13}. 
We denote the preconditioner as $\{\Gb_t \in \RR^{p\times p}\}$.
A popular choice in SG-MCMC is the Fisher information matrix \citep{GirolamiC:JRSSB11}.  
Unfortunately, this approach is computationally prohibitive for many models of interest.  
To avoid this problem, we adopt the preconditioner from RMSprop and Adam, which uses a vector
$\{\gb_t \in \RR^{p}\}$ to approximate the diagonal of the Fisher information matrixes~\citep{li2016preconditioned}.
The construction sequentially updates the preconditioner based on current and historical gradients with a 
smoothing parameter $\sigma$, and is shown as part of Algorithm \ref{alg:sahmc}.
While this approach will not capture the Riemannian geometry as effectively as the Fisher information matrix, 
it is computationally efficient.

Santa also introduces an annealing scheme on system temperatures.
As discussed in Section \ref{sec:preliminaries}, mSGNHT naturally accounts for a varying temperature by matching the particle momentum to the system temperature.  
We introduce $\beta = \{\beta_1, \beta_2, \cdots\}$, a sequence of inverse temperature variables with $\beta_i < \beta_j$ for $i < j$ and $\lim_{i \rightarrow \infty} \beta_i = \infty$.   
The infinite case corresponds to the zero-temperature limit, where SG-MCMCs become deterministic optimization methods.

The annealing scheme leads to two stages: the {\em exploration} and the {\em refinement} stages.
The {\em exploration} stage updates all parameters based on an annealed sequence of stochastic dynamic systems (see Section~\ref{sec:theory} for more details).  This stage is able to explore the parameter space efficiently, escape poor local modes, and finally
converge close to the global mode.
The {\em refinement} stage corresponds to the temperature limit, {\it i.e.}, $\beta_n \rightarrow \infty$.  In the temperature limit, the momentum weight updates vanish and it becomes a stochastic optimization algorithm.

We propose two update schemes to solve the corresponding stochastic differential equations: the Euler scheme and the symmetric splitting scheme (SSS). The Euler scheme has simpler updates, as detailed in Algorithm~\ref{alg:sahmc}; while SSS endows increased accuracy \citep{ChenDC:NIPS15} with a slight increase in overhead computation, as shown in Algorithm~\ref{alg:sahmc_ssi}. Section~\ref{para:SSI} elaborates on the details of these two schemes.  We recommend the use of SSS, but the Euler scheme is simpler to implement and compare to known algorithms.


\begin{algorithm}[h]
\SetKwInOut{Input}{Input}
\caption{Santa with SSS}
\Input{$\eta_t$ (learning rate), $\sigma$, $\lambda$, $burnin$,
$\beta = \{\beta_1, \beta_2, \cdots\}\rightarrow \infty$,  $\{\zetab_t \in \RR^p\}\sim \mathcal{N}({\bf 0},\textbf{I}_p)$.}
Initialize $\thetab_{0}$, $\ub_{0} = \sqrt{\eta}\times\mathcal{N}(0, I)$, $\alphab_0 = \sqrt{\eta}C$, $\vb_0 = 0$ \;
\For {$t = 1, 2, \ldots $} {
Evaluate $\tilde{\fb}_t \triangleq \nabla_{\thetab} \tilde{U}(\thetab_{t-1})$ on the $t^\text{th}$ mini-batch\;
$\vb_t = \sigma \vb_{t-1} + \frac{1 - \sigma}{N^2}\tilde{\fb}_t\odot \tilde{\fb}_t$ \;
$\gb_t = 1\oslash\sqrt{\lambda + \sqrt{\vb_t}}$ \;
$\thetab_{t} = \thetab_{t-1} + \gb_t \odot \ub_{t-1} / 2$\;
\uIf{$t < burnin$}{
\tcc{{\em exploration}}
$\alphab_t = \alphab_{t-1} + \left(\ub_{t-1}\odot \ub_{t-1} - \eta/\beta_t\right) / 2$\;
$\ub_{t} = \exp\left(-\alphab_{t}/2\right)\odot \ub_{t-1}$\;
$\ub_{t} = \ub_{t} - \gb_t\odot \tilde{\fb}_t \eta + \sqrt{2\gb_{t-1}\eta/\beta_t}\odot \zetab_t$ \
$~~~~+ \eta/\beta_t\left(1 - \gb_{t-1}\oslash \gb_{t}\right)\oslash \ub_{t-1}$\; 
$\ub_{t} = \exp\left(-\alphab_{t}/2\right)\odot \ub_{t}$\;
$\alphab_t = \alphab_{t} + \left(\ub_t\odot \ub_t - \eta/\beta_t\right) / 2$\;
}
\Else{
\tcc{{\em refinement}}
$\alphab_{t} = \alphab_{t-1}$; 
~~~~~~~~~~$\ub_{t} = \exp\left(-\alphab_{t}/2\right)\odot \ub_{t-1}$\;
$\ub_{t} = \ub_{t} - \gb_t\odot \tilde{\fb}_t \eta$; 
~$\ub_{t} = \exp\left(-\alphab_{t}/2\right)\odot \ub_{t}$\;
}
$\thetab_{t} = \thetab_{t} + \gb_t\odot \ub_{t} / 2$\;
}
\label{alg:sahmc_ssi}
\end{algorithm}

\paragraph{Practical considerations}
According to Section~\ref{sec:theory}, the {\em exploration} 
stage helps the algorithm traverse the parameter
space following the posterior curve as accurate as possible. For
optimization, slightly biased samples do not affect the final solution.
As a result, the term consisting of $(1 - \gb_{t-1}\oslash \gb_{t})$
in the algorithm (which is an approximation term, see Section~\ref{sec:gradient_app}) is ignored. We found no 
decreasing performance in our experiments. Furthermore, the term
$\gb_{t-1}$ associated with the Gaussian noise could be replaced 
with a fixed constant without affecting the algorithm.

%% file: tex/theory_dec.tex
\section{Theoretical Foundation}\label{sec:theory}
In this section we present the stochastic differential equations (SDEs) that correspond to the Santa algorithm.  
We first introduce the general SDE framework, then describe the \textit{exploration} stage in Section \ref{sec:exploration} and the \textit{refinement} stage in Section \ref{sec:refinement}.  We give the convergence properties of the numerical scheme in Section \ref{sec:convergence}.  This theory uses tools from the SDE literature and extends the mSGNHT theory \citep{DingFBCSN:NIPS14,GanCHCC:icml15}.

The SDEs are presented with re-parameterized $\pb = \ub / \eta^{1/2}$, $\Xib = \mbox{diag}(\alphab) / \eta^{1/2}$, as in \citet{DingFBCSN:NIPS14}.
The SDEs describe the motion of a particle in a system where $\thetab$ is the location and $\pb$ 
is the momentum.

In mSGNHT, the particle is driven by a force $ -\nabla_{\thetab}{\tilde{U}}_t(\thetab)$ at time $t$.
The stationary distribution of $\thetab$ corresponds to the model posterior \citep{GanCHCC:icml15}.
Our critical extension is the use of Riemannian information geometry, important for fast convergence \citep{PattersonT:NIPS13}.
Given an inverse temperature $\beta$, the system is described by the following SDEs\footnote{We abuse notation for conciseness.  Here, $\nabla_{\thetab}\Gb(\thetab)$ is a vector with the $i$-th element being 
$\sum_j \nabla_{\thetab_j}\Gb_{ij}(\thetab)$.}:

\vspace{-0.4cm}
{\small\begin{align}\label{eq:srhmc2}
	&\left\{\begin{array}{ll}
	\mathrm{d}\thetab &= G_1(\thetab)\pb \mathrm{d}t \\
	\mathrm{d}\pb &= \left(-G_1(\thetab)\nabla_\thetab U(\thetab) - \Xib \pb + \frac{1}{\beta}\nabla_\thetab G_1(\thetab)\right. \\
	 &\hspace{-0.6cm}\left.+ G_1(\thetab)(\Xib - G_2(\thetab))\nabla_{\thetab}G_2(\thetab)\right) \mathrm{d}t + ({\frac{2}{\beta}G_2(\thetab)})^{\frac{1}{2}}\mathrm{d}\w \\
	\mathrm{d}\Xib &= \left(\Qb - \frac{1}{\beta}I\right)\mathrm{d}t~,
	\end{array}\right.
\end{align}}
\hspace{-0.2cm}where $\Qb = \mbox{diag}(\pb\odot \pb)$, 
$\w$ is standard Brownian motion, 
$\Gb_1(\thetab)$ encodes geometric information of the potential energy $U(\thetab)$, 
and $\Gb_2(\thetab)$ characterizes the manifold geometry of the Brownian motion. 
Note $\Gb_2(\thetab)$ may be the same as $\Gb_1(\thetab)$ for the same Riemannian manifold. 
We call $\Gb_1(\thetab)$ and $\Gb_2(\thetab)$ Riemannian metrics, 
which are commonly defined by the Fisher information matrix \citep{GirolamiC:JRSSB11}. 
We use the RMSprop preconditioner (with updates from Algorithm~\ref{alg:sahmc}) for computational feasibility.
Using the Fokker-Plank equation \citep{Risken:FPE89}, we show that the marginal stationary distribution of \eqref{eq:srhmc2} corresponds to the posterior distribution.

\begin{lemma}\label{lem:FP}
	Denote $\Ab:\Bb \triangleq \mbox{tr}\left\{\Ab^T\Bb\right\}$.The stationary distribution of \eqref{eq:srhmc2} is:
	$p_\beta(\thetab, \pb, \Xib) \propto$
\begin{align}\label{eq:staitonary_beta}
	e^{-\beta U(\thetab) - \frac{\beta}{2}\pb^T\pb
	- \frac{\beta}{2}\left(\Xib - G_2(\thetab)\right):\left(\Xib - G_2(\thetab)\right)}~.
\end{align}\vspace{-0.5cm}
\end{lemma}

An inverse temperature $\beta=1$ corresponds to the standard Bayesian posterior.

We note that $\pb$ in \eqref{eq:srhmc2} has additional dependencies on $\Gb_1$ and $\Gb_2$ compared to \citet{GanCHCC:icml15} that must be accounted for. 
$\Xib \pb$ introduces friction into the system so that the particle does not move too far away by the random force;
the terms $\nabla_{\thetab}G_1(\thetab)$ and $\nabla_{\thetab}G_2(\thetab)$ penalize the influences of the Riemannian metrics so that the stationary distribution remains invariant.


\subsection{Exploration}
\label{sec:exploration}
The first stage of Santa, {\em exploration}, explores the parameter space to obtain parameters near the global mode of an objective function\footnote{This requires an ergodic algorithm. 
While ergodicity is not straightforward to check, we follow 
most MCMC work and assume it holds in our algorithm.}.
This approach applies ideas from simulated annealing \citep{Kirkpatrick:science83}. 
Specifically, the inverse temperature $\beta$ is slowly annealed to temperature zero to freeze the particles at the global mode. 


Minimizing $U(\thetab)$ is equivalent to sampling from the zero-temperature limit 
$p_\beta(\thetab) \triangleq \frac{1}{Z_\beta}e^{-\beta U(\thetab)}$ (proportional to \eqref{eq:staitonary_beta}), with $Z_\beta$ being the normalization constant such that $p_\beta(\thetab)$ is a valid distribution.
We construct a Markov chain that sequentially transits from high temperatures to low temperatures. 
At the state equilibrium, the chain reaches the temperature limit with marginal stationary distribution $\rho_0(\thetab) \triangleq \lim_{\beta \rightarrow \infty}e^{-\beta U(\thetab)}$, a point mass\footnote{The sampler samples a uniform distribution over global modes, or a point mass if the mode is unique.  We assume uniqueness and say point mass for clarity henceforth.} located at the global mode of $U(\thetab)$.
Specifically, we first define a sequence of inverse temperatures, $(\beta_1, \beta_2, \cdots, \beta_L)$, such that $\beta_L$ is large enough\footnote{Due to numerical issues, it is impossible to set $\beta_L$ to infinity; we thus assign a large enough value for it and handle the infinity case in the {\em refinement} stage.}.
For each time $t$, we generate a sample according to the SDE system \eqref{eq:srhmc2} with temperature $\frac{1}{\beta_t}$, 
conditioned on the sample from the previous temperature, $\frac{1}{\beta_{t-1}}$. 
We call this procedure annealing thermostats to denote the analog to simulated annealing. 

\paragraph{Generating approximate samples}\label{para:SSI}

Generating exact samples from \eqref{eq:srhmc2} is infeasible for general models.
One well-known numerical approach is the Euler scheme in Algorithm~\ref{alg:sahmc}.
The Euler scheme is a 1st-order method with relatively high approximation error \citep{ChenDC:NIPS15}.
We increase accuracy by implementing the symmetric splitting scheme (SSS) \citep{ChenDC:NIPS15,li2016high}.
The idea of SSS is to split an infeasible SDE into several sub-SDEs, where each sub-SDE is analytically solvable;
approximate samples are generated by sequentially evolving parameters via these sub-SDEs. 
Specifically, in Santa, we split \eqref{eq:srhmc2} into the following three sub-SDEs:

\vspace{-0.5cm}{\small\begin{align*}
	&A: \left\{\begin{array}{ll}
	\mathrm{d}\thetab &= \Gb_1(\thetab)\pb \mathrm{d}t \\
	\mathrm{d}\pb &= 0 \\
	\mathrm{d}\Xi &= \left(\Qb - \frac{1}{\beta}I\right)\mathrm{d}t
	\end{array}\right.,
	B: \left\{\begin{array}{ll}
	\mathrm{d}\thetab &= 0 \\
	\mathrm{d}\pb &= - \Xib \pb \mathrm{d}t \\
	\mathrm{d}\Xib &= 0
	\end{array}\right., \\
	&O: \left\{\begin{array}{ll}
	\mathrm{d}\thetab &= 0 \\
	\mathrm{d}\pb &= \left(-G_1(\thetab)\nabla_\thetab U(\thetab) + \frac{1}{\beta}\nabla_\thetab G_1(\thetab)\right. \\ 
	 &\hspace{-1cm}\left.+ G_1(\thetab)(\Xib - G_2(\thetab))\nabla_{\thetab}G_2(\thetab)\right) \mathrm{d}t + ({\frac{2}{\beta}G_2(\thetab)})^{\frac{1}{2}}\mathrm{d}\w \\
	\mathrm{d}\Xib &= 0
	\end{array}\right. \nonumber
\end{align*}}
We then update the sub-SDEs in order $A$-$B$-$O$-$B$-$A$ to generative approximate samples \citep{ChenDC:NIPS15}.  This uses half-steps $h/2$ on the $A$ and $B$ updates\footnote{As in \citet{DingFBCSN:NIPS14}, we define $h = \sqrt{\eta}$.}, and full steps $h$ in the $O$ update.  This is analogous to the leapfrog steps in Hamiltonian Monte Carlo \citep{neal2011mcmc}.  
Update equations are given in the Supplementary Section~\ref{sec:sde-solution}. The resulting parameters then serve as an approximate sample from the posterior distribution with the inverse temperature of $\beta$.
Replacing $\Gb_1$ and $\Gb_2$ with the RMSprop preconditioners gives Algorithm~\ref{alg:sahmc_ssi}. 
These updates require approximations to $\nabla_{\thetab}G_1(\thetab)$ and $\nabla_{\thetab}G_2(\thetab)$, addressed below.


\paragraph{Approximate calculation for $\nabla_{\thetab}\Gb_1(\thetab)$}\label{sec:gradient_app}
We propose a computationally efficient approximation for calculating the derivative vector $\nabla_{\thetab}\Gb_1(\thetab)$ 
based on the definition. Specifically, for the $i$-th element of $\nabla_{\thetab}\Gb_1(\thetab)$ at the $t$-th iteration, 
denoted as $(\nabla_{\thetab}\Gb_1^t(\thetab))_i$, it is approximated as:
{\small\begin{align*}
	(&\nabla_{\thetab}\Gb_1^t(\thetab))_i \stackrel{A_1}{\approx} \sum_j\frac{(\Gb_1^t(\thetab))_{ij} - (\Gb_1^{t-1}(\thetab))_{ij}}{\thetab_{tj} - \thetab_{(t-1)j}} \\
	&\stackrel{A_2}{=} \sum_j\frac{(\Delta \Gb_1^t)_{ij}}{(\Gb_1^t(\thetab)\pb_{t-1})_{j} h} = \sum_j\frac{(\Delta \Gb_1^t)_{ij}}{(\Gb_1^t(\thetab)\ub_{t-1})_{j}}
\end{align*} }
\hspace{-0.15cm}where $\Delta \Gb_1^t \triangleq \Gb_1^t - \Gb_1^{t-1}$. Step $A_1$ follows by the definition of a derivative, and $A_2$ by using the update equation for
$\thetab_t$, {\it i.e.}, $\thetab_t = \thetab_{t-1} + \Gb_1^t(\thetab)\pb_{t-1} h$.
According to Taylor's theory, the approximation error for the $\nabla_{\thetab}\Gb_1(\thetab)$ is $O(h)$, {\it e.g.},
{\small\begin{align}\label{eq:approx_grad}
	\sum_i\left| \sum_j\frac{(\Delta \Gb_1^t)_{ij}}{(\Gb_1^t(\thetab)\ub_{t-1})_{j}} - (\nabla_{\thetab}\Gb_1^t(\thetab))_i\right| \leq \mathcal{B}_t h~,
\end{align}}
\hspace{-0.15cm}for some positive constant $\mathcal{B}_t$. The approximation error is negligible in term of 
convergence behaviors because it can be absorbed into the stochastic gradients error. 
Formal theoretical analysis on convergence behaviors with this approximation is given in later sections. 
Using similar methods, $\nabla_{\thetab}G_2^t(\thetab)$ is also approximately calculated.

\subsection{Refinement}
\label{sec:refinement}
The {\em refinement} stage corresponds to the zero-temperature limit of the {\em exploration} stage, where $\Xib$ is learned.
We show that in the limit Santa gives significantly simplified updates, leading to an stochastic optimization algorithm similar to Adam or SGD-M.

We assume that the Markov chain has reached its equilibrium after the {\em exploration} stage.
In the zero-temperature limit, some terms in the SDE \eqref{eq:srhmc2} vanish. 
First, as $\beta \rightarrow \infty$, the term 
$\frac{1}{\beta}\nabla_{\thetab}\Gb_1(\thetab)$ and the variance term for the Brownian motion approach 0. 
As well, the thermostat variable $\Xib$ approaches $\Gb_2(\thetab)$, so the term
$\Gb_1(\thetab)(\Xib - \Gb_2(\thetab))\nabla_{\thetab}\Gb_2(\thetab)$ vanishes.
The stationary distribution in \eqref{eq:staitonary_beta} implies $\mathbb{E}\Qb_{ii} \triangleq \mathbb{E}\pb_i^2 \rightarrow 0$, which makes the
SDE for $\Xib$ in \eqref{eq:srhmc2} vanish.
As a result, in the {\em refinement} stage, only $\thetab$ and $\pb$ need to be updated. 
The Euler scheme for this is shown in Algorithm~\ref{alg:sahmc}, and the symmetric splitting scheme is shown in Algorithm~\ref{alg:sahmc_ssi}.

\paragraph{Relation to stochastic optimization algorithms}
In the \textit{refinement} stage Santa is a stochastic optimization algorithm.  This relation is easier seen with the Euler scheme in Algorithm \ref{alg:sahmc}.
Compared with SGD-M \citep{RumelhartHW:nature86}, Santa has both adaptive gradient and adaptive momentum updates.
Unlike Adagrad \citep{duchi2011adaptive} and RMSprop \citep{tieleman2012lecture}, \textit{refinement} Santa is a momentum based algorithm.

The recently proposed Adam algorithm \citep{kingma2014adam} incorporates momentum and preconditioning in what is denoted as ``adaptive moments."  We show in Supplementary Section \ref{sec:relate_adam} that a constant step size combined with a change of variables nearly recovers the Adam algorithm with element-wise momentum weights.  
For these reasons, Santa serves as a more general stochastic optimization algorithm that extends all current algorithms.  As well, for a convex problem and a few trivial algorithmic changes, the regret bound of Adam holds for \textit{refinement} Santa, which is $\mathcal{O}(\sqrt{T})$, as detailed in Supplementary Section \ref{sec:relate_adam}.  However, our analysis is focused on non-convex problems that do not fit in the regret bound formulation.
%
%
%

\subsection{Convergence properties}
\label{sec:convergence}
Our convergence properties are based on the framework of \citet{ChenDC:NIPS15}.  The proofs for all theorems are given in the Supplementary Material.
We focus on the {\em exploration} stage of the algorithm.  Using the Monotone Convergence argument \citep{Schechter:book97}, the {\em refinement} stage convergence is obtained by taking the temperature limit from the results of the {\em exploration} stage.
We emphasize that our approach differs from conventional stochastic optimization or online optimization approaches.  
Our convergence rate is weaker than many stochastic optimization methods, including SGD; however, our analysis applies to \textit{non-convex} problems, whereas traditionally convergence rates only apply to convex problems.


The goal of Santa is to obtain $\thetab^*$ such that $\thetab^* = \arg\!\min_{\thetab} U(\thetab)$.
Let $\{\thetab_1, \cdots, \thetab_L\}$ be a sequence of parameters collected from the algorithm. 
Define $\hat{U} \triangleq \frac{1}{L}\sum_{t=1}^L U(\thetab_t)$ as the
sample average, $\bar{U} \triangleq U(\thetab^*)$ the global optima of $U(\thetab)$.

As in \citet{ChenDC:NIPS15}, we require certain assumptions on the potential energy $U$.  To show these assumptions, we first define a functional $\psi_t$ for each $t$ that solves the following 
Poisson equation:
\begin{align}\label{eq:PoissonEq1}
	\mathcal{L}_t \psi_t(\thetab_{t}) =  U(\thetab_{t}) - \bar{U}~,
\end{align}
$\mathcal{L}_t$ is the generator of the SDE system \eqref{eq:srhmc2} in the $t$-th iteration, defined
$\mathcal{L}_tf(\xb_t) \triangleq \lim_{h \rightarrow 0^{+}} \frac{\mathbb{E}\left[f(\xb_{t+h})\right] - f(\xb_t)}{h}$
where $\xb_t \triangleq (\thetab_t, \pb_t, \Xib_t)$, $f: \RR^{3p} \rightarrow \RR$ is a compactly supported twice differentiable function.
The solution functional $\psi_t(\thetab_{t})$ characterizes the difference between $U(\thetab_{t})$ 
and the global optima $\bar{U}$ for every $\thetab_{t}$. As shown in \citet{MattinglyST:JNA10}, \eqref{eq:PoissonEq1} 
typically possesses a unique solution, which is at least as smooth as $U$ under the elliptic or hypoelliptic settings.
We assume $\psi_t$ is bounded and smooth, as described below.

\begin{assumption}\label{ass:assumption1}
$\psi_t$ and its up to 3rd-order derivatives, $\mathcal{D}^k \psi_t$, are bounded by a
function $\mathcal{V}(\thetab, \pb, \Xib)$, {\it i.e.}, 
$\|\mathcal{D}^k \psi\| \leq C_k\mathcal{V}^{r_k}$ for $k=(0, 1, 2, 3)$, $C_k, r_k > 0$. Furthermore, 
the expectation of $\mathcal{V}$ is bounded: $\sup_t \mathbb{E}\mathcal{V}^r(\thetab, \pb, \Xib) < \infty$, 
and $\mathcal{V}$ is smooth such that 
$\sup_{s \in (0, 1)} \mathcal{V}^r\left(s\x + \left(1-s\right)\y\right) \leq C\left(\mathcal{V}^r\left(\x\right) + \mathcal{V}^r\left(\y\right)\right)$, $\forall \x \in \RR^{3p}, \y \in \RR^{3p}, r \leq \max\{2r_k\}$ for some $C > 0$.
\end{assumption}

Let $\Delta U(\thetab) \triangleq U(\thetab) - U(\thetab^*)$. Further define an operator 
$\Delta V_t = \left(G_1(\thetab)(\nabla_{\thetab}\tilde{U}_t - \nabla_{\thetab} U) + \frac{\mathcal{B}_t}{\beta_t}\right) \cdot \nabla_{\pb}$
for each $t$, where $\mathcal{B}_t$ is from \eqref{eq:approx_grad}. Theorem~\ref{theo:bias1} depicts the closeness of $\hat{U}$ to
the global optima $\bar{U}$ in term of {\em bias} and {\em mean square error} (MSE) defined below.

\begin{theorem}\label{theo:bias1}
	 Let $\left\|\cdot\right\|$ be the operator norm.
	Under Assumption~\ref{ass:assumption1}, the bias and MSE of the {\em exploration} stage in Santa with respect to the global optima for $L$ steps
	with stepsize $h$ is bounded, for some constants $C > 0$ and $D > 0$, with:
	\vspace{-0.3cm}{\small\begin{align*}
		\mbox{Bias: }&\left|\mathbb{E}\hat{U} - \bar{U}\right| \leq Ce^{-U(\thetab^*)}\left(\frac{1}{L}\sum_{t=1}^L\int e^{-\beta_t\Delta U(\thetab)}\mathrm{d}\thetab\right) \\
		&~~~~~~~~~~~~~~~+ D\left(\frac{1}{Lh} + \frac{\sum_t \left\|\mathbb{E}\Delta V_t\right\|}{L} + h^2\right)~. \\
	\mbox{MSE: }&\mathbb{E}\left(\hat{U} - \bar{U}\right)^2 \leq C^2e^{-2U(\thetab^*)}\left(\frac{1}{L}\sum_{t=1}^L\int e^{-\beta_t\Delta U(\thetab)}\mathrm{d}\thetab\right)^2 \\
		&~~~~~~~~~~~~~~~+ D^2 \left(\frac{\frac{1}{L}\sum_t\mathbb{E}\left\|\Delta V_t\right\|^2}{L} + \frac{1}{Lh} + h^{4}\right)~.
	\end{align*}}
\end{theorem}\vspace{-0.2cm}

Both bounds for the bias and MSE have two parts. The first part contains integration terms, which characterizes the distance between
the global optima, $e^{-U(\thetab^*)}$, and the unnormalized annealing distributions, $e^{-\beta_t U(\thetab)}$, 
decreasing to zero exponentially fast with increasing $\beta$; 
the remaining part characterizes the distance between the sample average and the annealing posterior average. This shares a
similar form as in general SG-MCMC algorithms \citep{ChenDC:NIPS15}, and can be controlled to converge.
Furthermore, the term $\frac{\sum_t \left\|\mathbb{E}\Delta V_t\right\|}{L}$ in the bias vanishes as long as the sum of the annealing sequence $\{\beta_t\}$ is finite\footnote{In practice we might
not need to care about this constraint because a small bias in the {\em exploration} stage does not affect convergence of
the {\em refinement} stage.}, indicating that the gradient
approximation for $\nabla_{\thetab}G_1(\thetab)$ in Section~\ref{sec:gradient_app} does not affect the bias of the algorithm. 
Similar arguments apply for the MSE bound.

To get convergence results right before the {\em refinement} stage, let a sequence
of functions $\{g_m\}$ be defined as $g_m \triangleq -\frac{1}{L}\sum_{l=m}^{L+m-1} e^{-\beta_l\hat{U}(\thetab)}$;
it is easy to see that $\{g_m\}$ satisfies $g_{m_1} < g_{m_2}$ for $m_1 < m_2$, and $\lim_{m \rightarrow \infty} g_m = 0$. 
According to the Monotone Convergence Theorem \citep{Schechter:book97}, the bias and MSE in
the limit exists, leading to Corollary~\ref{coro:refine_con}.

\begin{corollary}\label{coro:refine_con}
	Under Assumptions~\ref{ass:assumption1}, the bias and MSE of the {\em refinement} stage in Santa with 
	respect to the global optima for $L$ steps with stepsize $h$ are bounded, for some constants 
	$D_1 > 0$, $D_2 > 0$, as
	\begin{align*}
		\mbox{Bias: }&\left|\mathbb{E}\hat{U} - \bar{U}\right| \leq D_1\left(\frac{1}{Lh} + \frac{\sum_t \left\|\mathbb{E}\Delta V_t\right\|}{L} + h^2\right) \\
		\mbox{MSE: }&\mathbb{E}\left(\hat{U} - \bar{U}\right)^2 \leq D_2 \left(\frac{\frac{1}{L}\sum_t\mathbb{E}\left\|\Delta V_t\right\|^2}{L} + \frac{1}{Lh} + h^{4}\right)
	\end{align*}
\end{corollary}

Corollary~\ref{coro:refine_con} implies that in the {\em refinement} stage, the discrepancy between
annealing distributions and the global optima vanishes, leaving only errors from discretized simulations
of the SDEs, similar to the result of general SG-MCMC \citep{ChenDC:NIPS15}. We note that after {\em exploration}, 
Santa becomes a pure stochastic optimization algorithm, thus convergence results in term of regret bounds can 
also be derived; refer to Supplementary Section~\ref{sec:relate_adam} for more details.

%% file: tex/experiments.tex
\section{Experiments}\label{sec:experiments}
\subsection{Illustration}
In order to demonstrate that Santa is able to achieve the global mode of an objective function, we consider the
double-well potential \citep{DingFBCSN:NIPS14},
\begin{align*}
	U(\theta) = (\theta + 4)(\theta + 1)(\theta - 1)(\theta - 3) / 14 + 0.5~.
\end{align*}
As shown in Figure~\ref{fig:syn_twowell} (left), the double-well potential has two modes, located
at $\theta = -3$ and $\theta = 2$, with the global optima at $\theta = -3$.
We use a decreasing learning rate $h_t = t^{-0.3} / 10$, and the annealing sequence
is set to $\beta_t = t^{2}$. To make the optimization more challenging, we initialize the
parameter at $\theta_0 = 4$, close to the local mode. The evolution of $\theta$ with respect to iterations
is shown in Figure~\ref{fig:syn_twowell}(right). As can be seen, $\theta$ first moves to the local mode but quickly jumps
out and moves to the global mode in the {\em exploration} stage (first half iterations); 
in the {\em refinement} stage, $\theta$ quickly converges to the global mode and sticks to it afterwards. 
In contrast, RMSprop is trapped on the local optima, and convergences slower than Santa at the beginning.
\begin{figure}[h] 
	\centering
	\begin{minipage}{0.49\linewidth}
		\includegraphics[width=\linewidth]{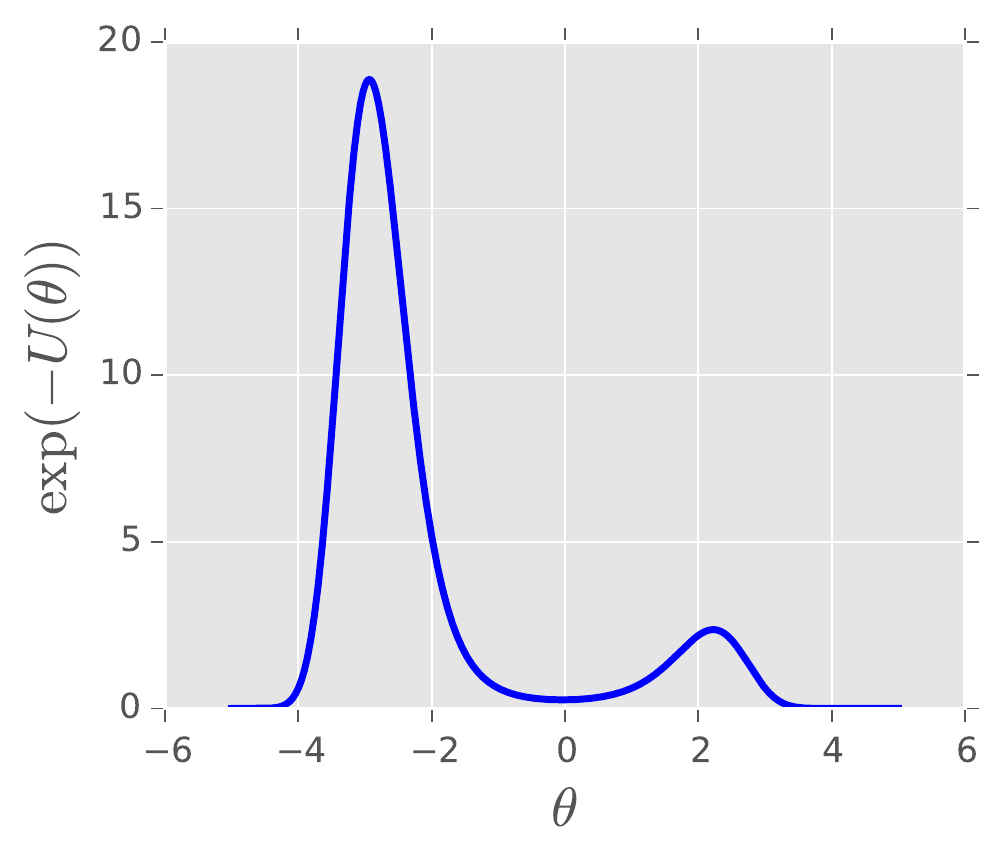}
	\end{minipage}
	\begin{minipage}{0.49\linewidth}
		\includegraphics[width=\linewidth]{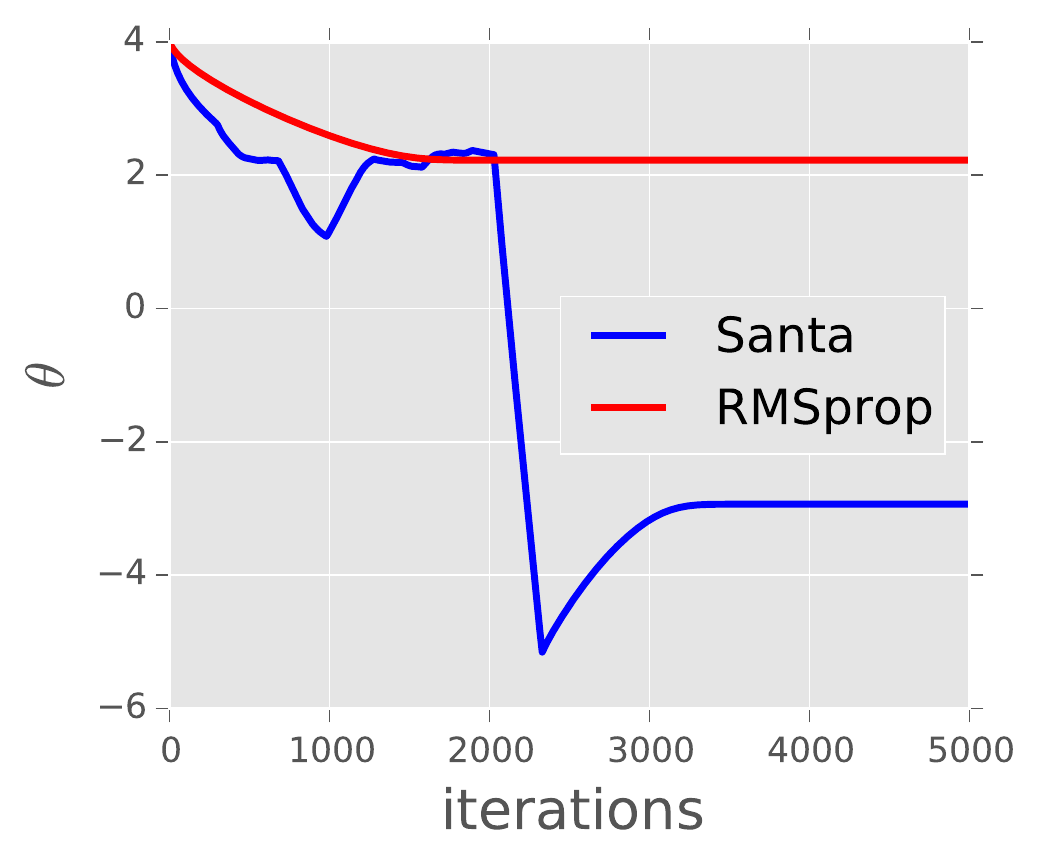}
	\end{minipage}
	\caption{(Left) Double-well potential. (Right) The evolution of $\theta$ using Santa and RMSprop algorithms.}
	\label{fig:syn_twowell}
\end{figure}
\subsection{Feedforward neural networks}
We first test Santa on the Feedforward Neural Network (FNN) with rectified linear units (ReLU). We test two-layer models with network sizes 784-X-X-10, where X is the number of hidden units for each layer; 100 epochs are used. For variants of Santa, we denote Santa-E as Santa with a Euler scheme illustrated in Algorithm~\ref{alg:sahmc}, Santa-r as Santa running only on the {\em refinement} stage, but with updates on $\alphab$ as in the {\em exploration} stage. We compare Santa with SGD, SGD-M, RMSprop, Adam, SGD with dropout, SGLD and Bayes by Backprop \citep{blundell2015weight}. We use a grid search to obtain good learning rates for each algorithm, resulting in $4\times 10^{-6}$ for Santa, $5\times 10^{-4}$ for RMSprop, $10^{-3}$ for Adam, and $5 \times 10^{-1} $ for SGD, SGD-M and SGLD. We choose an annealing schedule of $\beta_t = A t^{\gamma}$ with $A = 1$ and $\gamma$ selected from 0.1 to 1 with an interval of 0.1. For simplicity,
the {\em exploration} is set to take half of total iterations.

We test the algorithms on the standard MNIST dataset, which contains $28 \times 28$ handwritten digital images from $10$ classes with $60,000$ training samples and $10,000$ test samples. The network size (X-X) is set to 400-400 and 800-800, and test classification errors are shown in Table~\ref{tab:fnn_cnn}.  Santa show improved state-of-the-art performance amongst all algorithms. The Euler scheme shows a slight decrease in performance, due to the integration error when solving the SDE. Santa without {\em exploration} (\emph{i.e.}, Santa-r) still performs relatively well. Learning curves are plotted in Figure~\ref{fig:fnn_cnn}, showing that Santa converges as fast as other algorithms but to a better local optima\footnote{Learning curves of FNN with size of 800 are provided in Supplementary Section~\ref{supp:additional_results}.}.
\subsection{Convolution neural networks}
\begin{figure}[t!]
	\centering
	\includegraphics[width=0.49\linewidth]{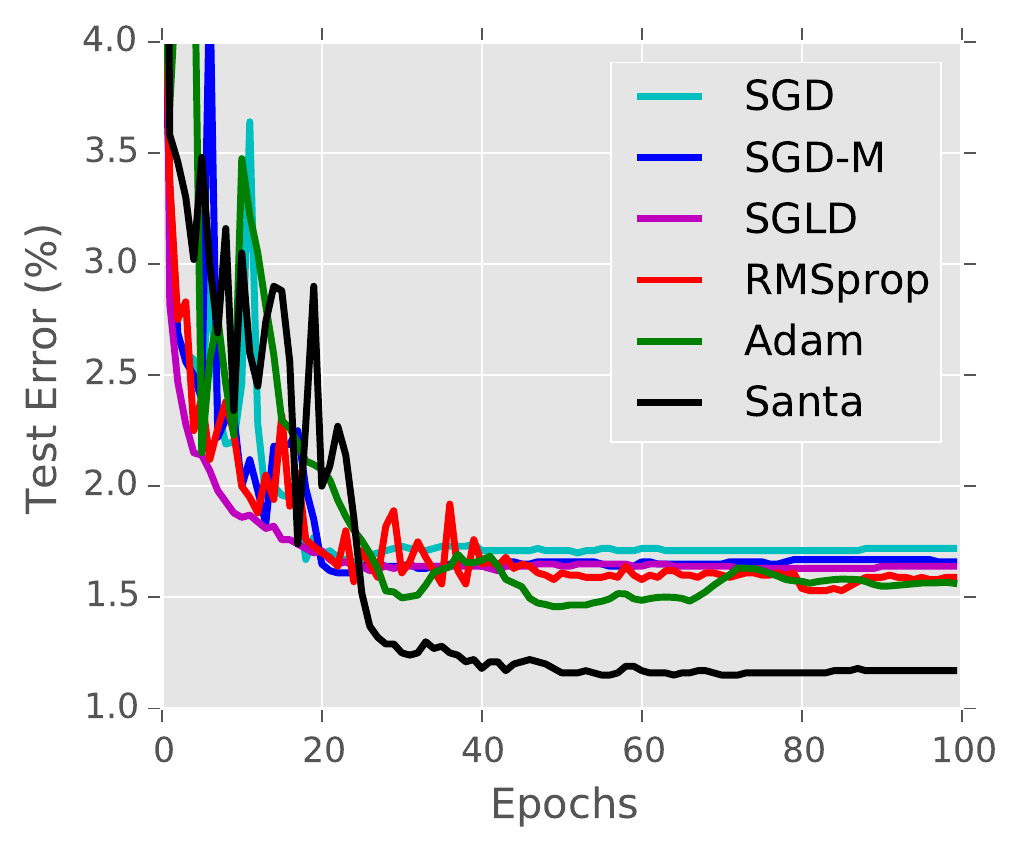}
    \includegraphics[width=0.49\linewidth]{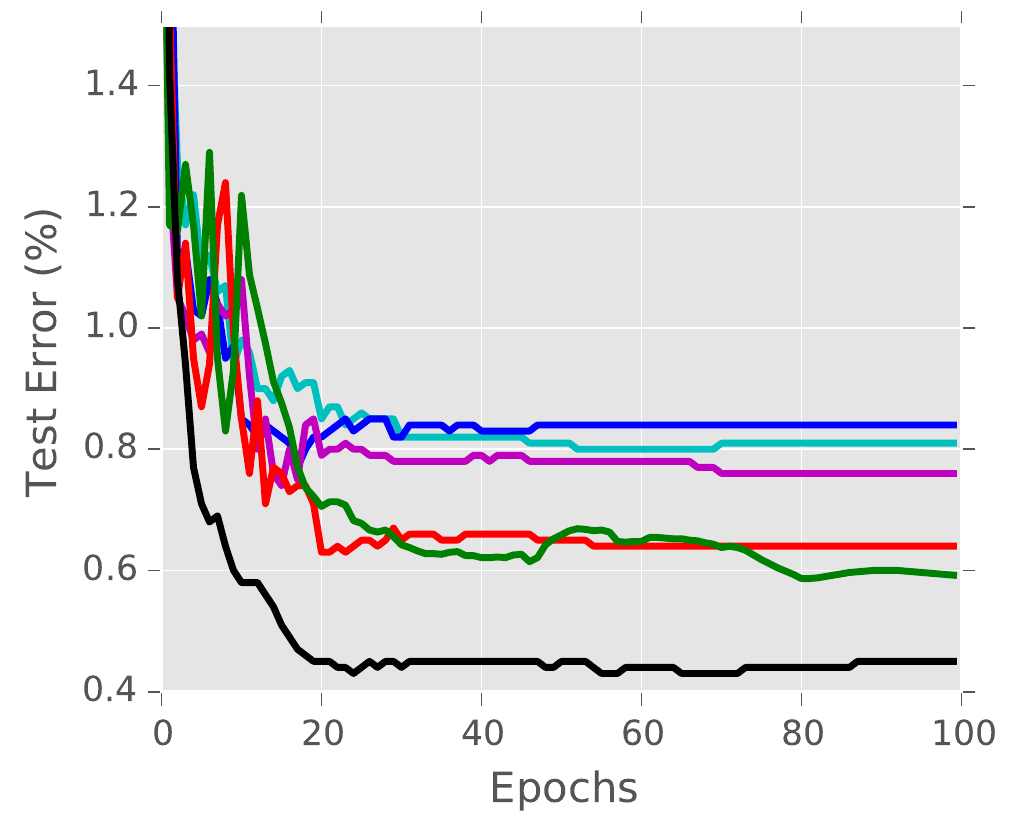}
	\caption{\small{Learning curves of different algorithms on MNIST. (Left) FNN with size of 400.  (Right) CNN.}}
	\label{fig:fnn_cnn}
\end{figure}
\begin{table}[t!]
	\caption{\small{Test error on MNIST classification using FNN and CNN. ($^\diamond$) taken from \citet{blundell2015weight}. ($^\triangleright$) taken from \citet{zeiler2013stochastic}. ($^\circ$) taken from \citet{LinCY14NIN}. ($^\star$) taken from \citet{goodfellow2013maxout}.} }\label{tab:fnn_cnn}
  	\centering
  	\begin{adjustbox}{minipage=1.05\linewidth,scale=0.95}
  	\begin{tabular}{c c c c} \hline
    		Algorithms & FNN-400 & FNN-800 & CNN \\
    		\hline
    		Santa & {\bf1.21\%} & {\bf1.16\%} & {\bf0.47\%} \\
    		Santa-E & 1.41\% & 1.27\% & 0.58\% \\
    		Santa-r & 1.45\% & 1.40\% & 0.49\% \\
    		\hline
    		Adam & 1.53\% & 1.47\% & 0.59\% \\
    		RMSprop & 1.59\% & 1.43\% & 0.64\% \\
    		SGD-M & 1.66\% & 1.72\% & 0.77\% \\
    		SGD & 1.72\% &1.47\% & 0.81\% \\
    		SGLD & 1.64\% & 1.41\% & 0.71\% \\
    		\hline
    		BPB$^\diamond$ & 1.32\% & 1.34\% & $-$\\
    		SGD, Dropout$^\diamond$ & 1.51\% &1.33\% & $-$ \\
    		Stoc. Pooling$^\triangleright$ & $-$ & $-$ & 0.47\% \\
		    NIN, Dropout$^\circ$ & $-$ & $-$ & 0.47\% \\
		    Maxout, Dropout$^\star$& $-$ & $-$  & 0.45\% \\
    		\hline
  	\end{tabular}
  	\end{adjustbox}
\end{table}
We next test Santa on the Convolution Neural Network (CNN). Following \citet{jarrett2009best}, a standard 
network configuration with 2 convolutional layers followed by 2 fully-connected layers is adopted. Both convolutional layers use $5 \times 5$ filter size with 32 and 64 channels, respectively; $2 \times 2$ max pooling is used after each convolutional layer. The fully-connected layers have 200-200 hidden nodes with ReLU activation. The same parameter setting and dataset as in the FNN are used. The test errors are shown in Table~\ref{tab:fnn_cnn}, and the corresponding learning curves are shown in Figure~\ref{fig:fnn_cnn}. Similar trends as in FNN are obtained. Santa significantly outperforms other algorithms with an error of 0.45\%. This result is comparable or even better than some recent state-of-the-art CNN-based systems, which have much more complex architectures.
\subsection{Recurrent neural networks}
We test Santa on the Recurrent Neural Network (RNN) for sequence modeling, where a model is trained to minimize the negative log-likelihood of training sequences:
\begin{align} \label{eq:rnn}
\min_{\vtheta} \frac{1}{N} \sum_{n=1}^N \sum_{t=1}^{T_n} -\log p(\xb_t^n|\xb_1^n,\ldots,\xb_{t-1}^n;\vtheta)
\end{align}
where $\vtheta$ is a set of model parameters, $\{\xb_t^n\}$ is the observed data. The conditional distributions in (\ref{eq:rnn}) are modeled by the RNN. The hidden units are set to gated recurrent units \citep{cho2014learning}. 

We consider the task of sequence modeling on four different polyphonic music sequences of piano, \emph{i.e.}, Piano-midi.de (Piano), Nottingham (Nott), MuseData (Muse) and JSB chorales (JSB). Each of these datasets are represented as a collection of 88-dimensional binary sequences, that span the whole range of piano from A0 to C8. 

The number of hidden units is set to 200. Each model is trained for at most 100 epochs. According to the experiments and their results on the validation set, we use a learning rate of 0.001 for all the algorithms. For Santa, we consider an additional experiment using a learning rate of 0.0002, denoted Santa-s. The annealing coefficient $\gamma$ is set to 0.5. Gradients are clipped if the norm of the parameter vector exceeds 5. We do not perform any dataset-specific tuning other than early stopping on validation sets. Each update is done using a minibatch of one sequence. 
%
%
\begin{figure}[t!]
	\centering
    \includegraphics[width=0.49\linewidth]{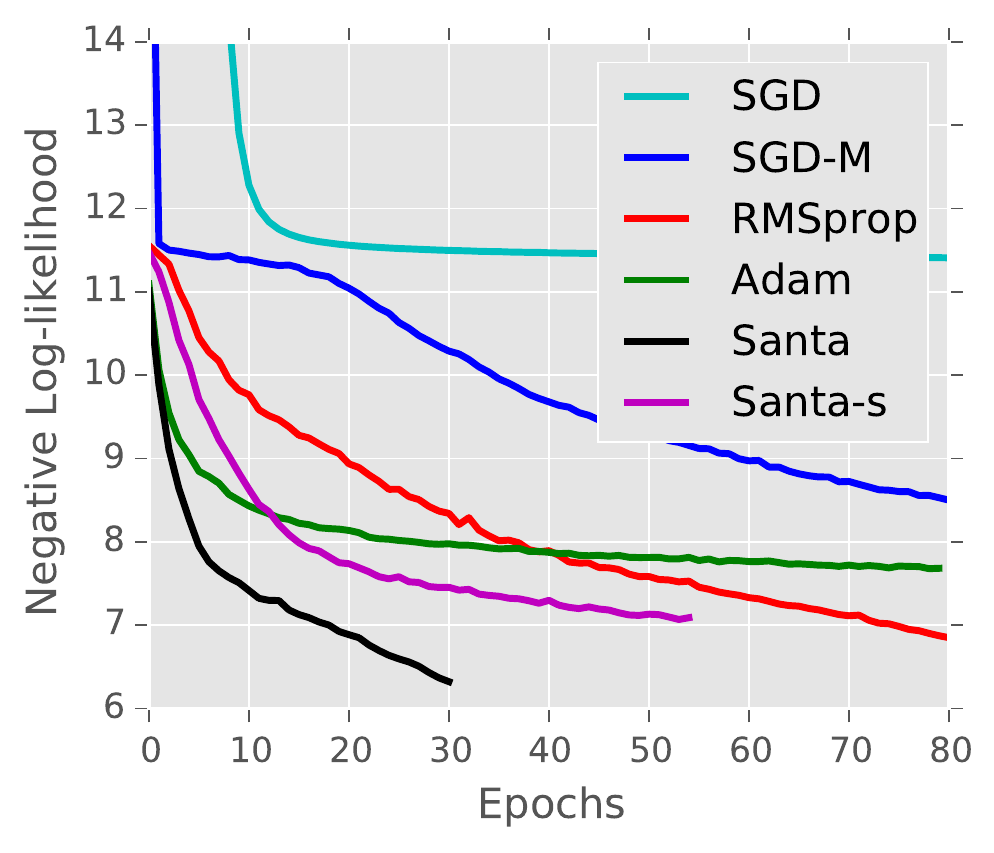}
    \includegraphics[width=0.49\linewidth]{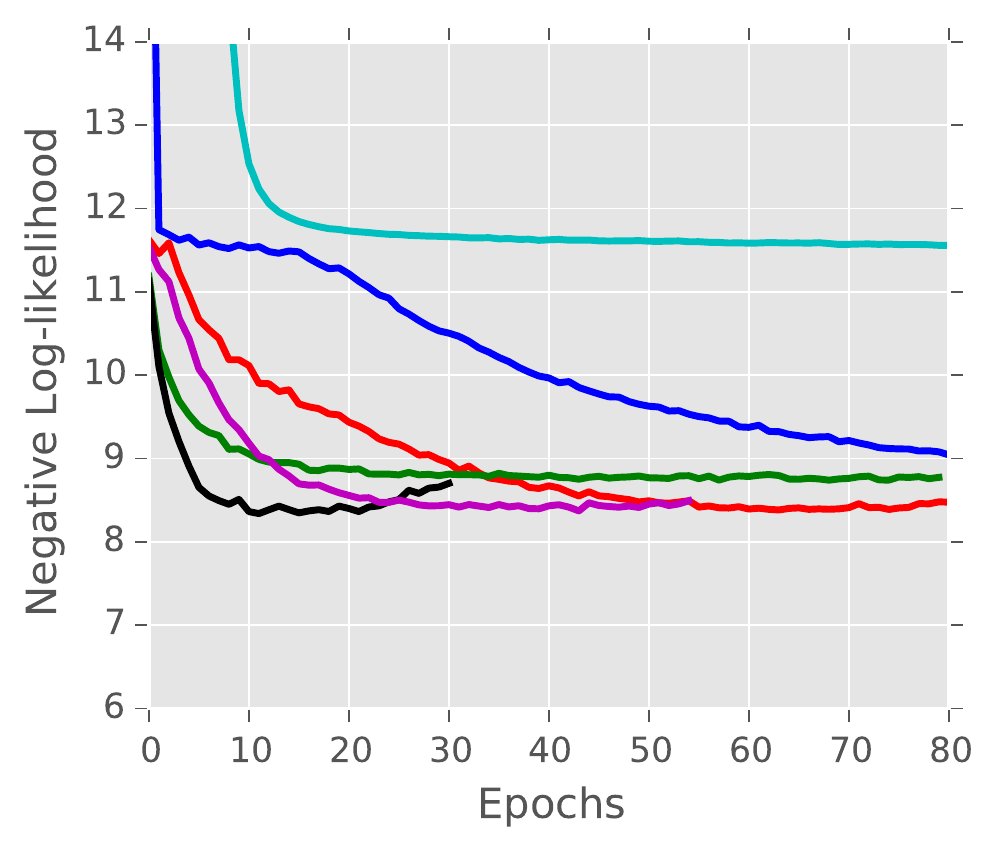}
	\caption{\small{Learning curves of different algorithms on Piano using RNN. (Left) training set. (Right) validation set.}}
	\label{fig:rnn}
\end{figure}
\begin{table}[t!] 
	\caption{Test negative log-likelihood results on polyphonic music datasets using RNN. ($^\diamond$) taken from \citet{boulanger2012modeling}.}\label{tab:rnn}
  	\centering
  	\begin{tabular}{c c c c c} \hline
    		Algorithms & Piano. & Nott. & Muse. & JSB. \\
    		\hline
    		Santa & {\bf7.60} & {\bf3.39} & {\bf7.20} & {\bf8.46} \\
    		Adam & 8.00 & 3.70 & 7.56 & 8.51 \\
    		RMSprop & 7.70 & 3.48 & 7.22 & 8.52 \\
    		SGD-M & 8.32 & 3.60 & 7.69 & 8.59 \\
    		SGD & 11.13 & 5.26 & 10.08 & 10.81 \\
    		\hline
    		HF$^\diamond$ & 7.66 & 3.89 & {\bf7.19} & 8.58 \\
    		SGD-M$^\diamond$ & 8.37 & 4.46 & 8.13 & 8.71 \\
    		\hline
  	\end{tabular}
\end{table}

The best log-likelihood results on the test set are achieved by using Santa, shown in Table~\ref{tab:rnn}. Learning curves on the Piano dataset are plotted in Figure~\ref{fig:rnn}. We observe that Santa achieves fast convergence, but is overfitting. This is straightforwardly addressed through early stopping. The learning curves for all the other datasets are provided in Supplementary Section~\ref{supp:additional_results}.

%% file: tex/conclusions.tex
\section{Conclusions}
We propose Santa, an annealed SG-MCMC method for stochastic optimization.  Santa is able to explore the parameter space efficiently and locate close to the
global optima by annealing. At the zero-temperature limit, Santa gives a novel stochastic optimization algorithm
where both model parameters and momentum are updated element-wise and adaptively. We provide theory on the convergence of Santa to the global
optima for an (non-convex) objective function. Experiments show best results on several deep models compared 
to related stochastic optimization algorithms.


%% file: tex/ack.tex
\section*{Acknowledgements}
This research was supported in part by ARO, DARPA,
DOE, NGA, ONR and NSF.

%% file: tex/appendix.tex
\section{Solutions for the sub-SDEs}\label{sec:sde-solution}

We provide analytic solutions for the split sub-SDEs in Section~\ref{para:SSI}. For stepsize $h$, the solutions 
are given in \eqref{eq:subsde_solution}.
\begin{align}\label{eq:subsde_solution}
	&A: \left\{\begin{array}{ll}
	\thetab_t &= \thetab_{t-1} + \Gb_1(\thetab)\pb h \\
	\pb_t &= \pb_{t-1} \\
	\Xib_{t} &= \Xib_{t-1} +\left(\Qb - \frac{1}{\beta}I\right)h
	\end{array}\right., \\
	&B: \left\{\begin{array}{ll}
	\thetab_{t} &= \thetab_{t-1} \\
	\pb_{t} &= \exp\left(- \Xib h\right) \pb_{t-1} \\
	\Xib_{t} &= \Xib_{t-1}
	\end{array}\right., \nonumber\\
	&O: \left\{\begin{array}{ll}
	\thetab_{t} &= \thetab_{t-1} \\
	\pb_{t} &= \pb_{t-1} + \left(-G_1(\thetab)\nabla_\thetab U(\thetab) + \frac{1}{\beta}\nabla_\thetab G_1(\thetab)\right. \\ 
	 &~~~~~~~~\left.+ G_1(\thetab)(\Xib - G_2(\thetab))\nabla_{\thetab}G_2(\thetab)\right) h \\
	 &~~~~~~~~+ ({\frac{2}{\beta}G_2(\thetab)})^{\frac{1}{2}}\odot \zeta_t \\
	\Xib_{t} &= \Xib_{t-1}
	\end{array}\right. \nonumber
\end{align}

\section{Proof of Lemma~\ref{lem:FP}}

For a general stochastic differential equation of the form
\begin{align}\label{eq:generalSDE}
	\mathrm{d} \xb = F(\xb) \mathrm{d}t + \sqrt{2} D^{1/2}(\xb) \mathrm{d} \wb~,
\end{align}
where $\xb \in \RR^N$,
$F: \RR^N \rightarrow \RR^N$, $D: \RR^{M} \rightarrow \RR^{N \times P}$ are measurable
functions with $P$, and $\wb$ is standard $P$-dimensional 
Brownian motion. \eqref{eq:srhmc2} is a special case of the general form \eqref{eq:generalSDE} with
\begin{align}\label{eq:reformulate}
\xb &= (\thetab, \pb, \Xib) \\
F(\xb) &= \left( \begin{array}{c}
\Gb_1(\thetab)\pb \\
-\Gb_1(\thetab)\nabla_{\thetab}U(\thetab) - \Xib\pb + \frac{1}{\beta}\nabla_{\thetab}\Gb_1(\thetab) \\
~~~~~+ \Gb_1(\thetab)(\Xib - \Gb_2(\thetab))\nabla_{\thetab}\Gb_2(\thetab) \\
\Qb - \frac{1}{\beta}\mathbf{I} \end{array} \right) \nonumber\\
D(\xb) &= \left( \begin{array}{ccc}
\mathbf{0} & \mathbf{0} & \mathbf{0} \\
\mathbf{0} & \frac{1}{\beta}\Gb_2(\thetab) & \mathbf{0} \\
\mathbf{0} & \mathbf{0} & \mathbf{0} \end{array} \right) \nonumber
\end{align}
We write the joint distribution of $\xb$ as
\begin{align*}
	\rho(\xb) &= \frac{1}{Z} \exp\left\{-H(\xb)\right\} \triangleq \frac{1}{Z} \exp\left\{-U(\thetab) - E(\thetab, \pb, \Xib)\right\}~.
\end{align*}

A reformulation of the main theorem in \citet{DingFBCSN:NIPS14} gives the following lemma, which is used to prove
Lemma~\ref{lem:FP} in the main text.
\begin{lemma} \label{lemma:fpe}
	The stochastic process of $\vec{\theta}$ generated by the stochastic differential equation~\eqref{eq:generalSDE} has the target distribution $p_{\thetab}(\thetab) = \frac{1}{Z} \exp\{-U(\thetab)\}$ as its stationary distribution, if $\rho(\xb)$ satisfies the following marginalization condition:
	\begin{align}\label{eq:margcond}
		\exp\{-U(\thetab)\} \propto \int \exp\{-U(\thetab) - E(\thetab, \pb, \Xib)\} \mathrm{d}\pb \mathrm{d}\Xib \,,
	\end{align}
	and if the following condition is also satisfied:
	\begin{align} \label{eq:fpecond}
		\nabla \cdot (\rho F) = \nabla \nabla^\top:(\rho D) \,,
	\end{align}
	where $\nabla \triangleq \left(\partial/\partial \thetab, \partial / \partial \pb, \partial / \Xib\right)$, ``$\cdot$'' represents the vector inner product operator, ``$:$'' represents a matrix double dot product, {\it i.e.}, $\Xb:\Yb \triangleq \mbox{tr}(\Xb^\top \Yb)$.
\end{lemma}

\begin{proof}[Proof of Lemma~\ref{lem:FP}]
	We first have reformulated \eqref{eq:srhmc2} using the general SDE form of \eqref{eq:generalSDE},
	resulting in \eqref{eq:reformulate}.
	Lemma~\ref{lem:FP} states the joint distribution of $(\thetab, \pb, \Xib)$ is
	\begin{align}\label{eq:equidis}
		\rho(\xb) =& \frac{1}{Z} \exp\left(-\frac{1}{2}\pb^\top \pb - U(\thetab)\right. \nonumber\\
			-& \left.\frac{1}{2} \mbox{tr} \left\{\left(\Xib - \Gb_2(\thetab)\right)^\top\left(\Xib - \Gb_2(\thetab)\right)\right\}\right)~,
	\end{align}
	with $H(\xb) = \frac{1}{2}\pb^\top \pb + U(\thetab) + \frac{1}{2} \mbox{tr} \left\{\left(\Xib - \Gb_2(\thetab)\right)^\top\left(\Xib - \Gb_2(\thetab)\right)\right\}$. The marginalization condition~\eqref{eq:margcond} is trivially satisfied, we are left to verify condition~\eqref{eq:fpecond}. Substituting $\rho(\xb)$ and $F$ into \eqref{eq:fpecond}, we have the left-hand side
	
	\begin{adjustbox}{minipage=1.0\linewidth,scale=0.88}
	\begin{align*}
		&\mbox{LHS} = \sum_i \frac{\partial}{\partial \xb_i} \left(\rho F_i\right) \\
		=& \sum_i \frac{\partial \rho}{\partial \xb_i} F_i + \frac{\partial F_i}{\partial \xb_i} \rho \\
		=& \sum_i \left(\frac{\partial F_i}{\partial \xb_i} - \frac{\partial H}{\partial \xb_i}F_i\right) \rho \\ 
		=& \left(\sum_i \nabla_{\thetab_i}(\Gb_1)_{i:}\pb - \sum_i\mbox{diag}(\Xib)\right. \\
		&- \sum_{i}\beta\left(\nabla_{\thetab_i}U - \sum_{j}(\Xib_{ij} - (\Gb_2)_{ij})\nabla_{\thetab_i}(\Gb_2)_{ij}\right)(\Gb_1\pb)_i \\
		&- \beta p^T\left(-\Gb_1\nabla_{\thetab}U - \Xib\pb + \frac{1}{\beta}\nabla_{\thetab}\Gb_1+ \Gb_1(\Xib - \Gb_2)\nabla_{\thetab}\Gb_2\right) \\
		&\left.- \beta \sum_i \left(\Xib_{ii} - (\Gb_2)_{ii}\right)\left(\Qb_{ii} - \frac{1}{\beta}\right)\right)\rho \\
		&= \frac{1}{\beta} \mbox{tr}\left\{\Gb_2(\pb\pb^T - \mathbf{I})\right\} \rho \,.
	\end{align*}
	\end{adjustbox}
	
	It is easy to see for the right-hand side
	\begin{align*}
		\mbox{RHS} &= \sum_i \sum_j \frac{1}{\beta}(\Gb_2)_{ij} \frac{\partial^2}{\partial \xb_i \partial \xb_j} \rho \\
		&= \frac{1}{\beta}\sum_i\sum_j (\Gb_2)_{ij}\frac{\partial}{\partial \pb_j} \left(- \frac{\partial H}{\partial \pb_i} \rho\right) \\
		&= \frac{1}{\beta}\sum_i (\Gb_2)_{ii}\left(\pb_i^2 - 1\right) \rho \\
		&\equiv \mbox{LHS} \,.
	\end{align*}
	According to Lemma~\ref{lemma:fpe}, the joint distribution~\eqref{eq:equidis} is the equilibrium distribution of \eqref{eq:srhmc2}.
\end{proof}

\section{Proof of Theorem~\ref{theo:bias1}}

We start by proving the bias result of Theorem~\ref{theo:bias1}.

\begin{proof}[Proof of the bias]

For our $2$nd-order integrator, according to the definition, we have: 
	\begin{align} \label{eq:split_flow}
		&\mathbb{E}[\psi(\Xb_{t})] = \tilde{P}_h^l \psi(\Xb_{t-1}) = e^{h\tilde{\Lcal}_t} \psi(\Xb_{t-1}) + O(h^{3}) \nonumber\\
		&= \left(\mathbb{I} + h\tilde{\Lcal}_t\right) \psi(\Xb_{t-1}) + \frac{h^2}{2}\tilde{\Lcal}_t^2\psi(\Xb_{t-1}) + O(h^{3})~,
	\end{align}
	where $\mathcal{L}_t$ is the generator of the SDE for the $t$-th iteration, {\it i.e.}, using stochastic gradient instead of the full
	gradient, $\mathbb{I}$ is the identity map. Compared to the prove of \citet{ChenDC:NIPS15}, we need to consider the 
	approximation error for $\nabla_{\thetab}G_1(\thetab)$. As a result, \eqref{eq:split_flow} needs to be rewritten as:
	\begin{align} \label{eq:split_flow1}
		&\mathbb{E}[\psi(\Xb_{t})] \\
		\hspace{-0.5cm}=& \left(\mathbb{I} + h(\tilde{\Lcal}_t + \mathcal{B}_t)\right) \psi(\Xb_{t-1}) + \frac{h^2}{2}\tilde{\Lcal}_t^2\psi(\Xb_{t-1}) + O(h^{3})~,\nonumber
	\end{align}
	where $\mathcal{B}_t$ is from \eqref{eq:approx_grad}.
	Sum over $t = 1, \cdots, L$ in \eqref{eq:split_flow1}, take expectation on both sides, and use the relation
	$\tilde{\Lcal}_t  + \mathcal{B}_t = \Lcal_{\beta_t} + \Delta V_t$ to expand the first order term. We obtain
	\begin{align*}
		&\sum_{t=1}^{L}\mathbb{E}[\psi(\Xb_{t})]
		= \psi(\Xb_0) + \sum_{t=1}^{L-1} \mathbb{E}[\psi(\Xb_{t})] \\
		&+ h\sum_{t=1}^{L} \mathbb{E}[\mathcal{L}_{\beta_t}\psi(\Xb_{t-1})]
		+ h\sum_{t=1}^L \mathbb{E}[\Delta V_t \psi(\Xb_{t-1})] \\
		&+ \frac{h^2}{2}\sum_{t=1}^L \EE[\tilde{\Lcal}_t^2\psi(\Xb_{t-1})] + O(L h^{3}).
	\end{align*}
	We divide both sides by $Lh$, use the Poisson equation \eqref{eq:PoissonEq1}, and reorganize terms. We have:
	\begin{align}\label{eq:expansion11}
		&\mathbb{E}[\frac{1}{L}\sum_t\left(\phi(\Xb_{t}) - \bar{\phi}_{\beta_t}\right)] = \frac{1}{L}\sum_{t=1}^{L} \mathbb{E}[\mathcal{L}_{\beta_t}\psi(\Xb_{t-1})] \nonumber\\
		=&\frac{1}{Lh}\left(\mathbb{E}[\psi(\Xb_{t})] - \psi(\Xb_0)\right)
		- \frac{1}{L}\sum_t \mathbb{E}[\Delta V_t\psi(\Xb_{t-1})] \nonumber\\
		&- \frac{h}{2L}\sum_{t=1}^L \EE[\tilde{\Lcal}_t^2\psi(\Xb_{t-1})] + O(h^2) 
	\end{align}
	Now we try to bound $\tilde{\Lcal}_t^2$. Based on ideas from \citet{MattinglyST:JNA10},
	we apply the following procedure.
	First replace $\psi$ with $\tilde{\Lcal}_t\psi$ from 
	\eqref{eq:split_flow1} to \eqref{eq:expansion11}, and apply the same logic for $\tilde{\Lcal}_t\psi$ 
	as for $\psi$ in the above derivations, but this time expand in \eqref{eq:split_flow1}
	up to the order of $O(h^2)$, instead of the previous order $O(h^{3})$. After simplification, we obtain:
	\begin{align}\label{eq:expansion21}
		&\sum_t \mathbb{E}[\tilde{\Lcal}_t^2 \psi(\Xb_{t-1})]
		= O\left(\frac{1}{h} + Lh\right)
	\end{align}
	Substituting \eqref{eq:expansion21} into \eqref{eq:expansion11}, after simplification, we have:
	$\mathbb{E}\left(\frac{1}{L}\sum_t\left(\phi(\Xb_{t}) - \bar{\phi}_{\beta_t}\right)\right)$
	\begin{align*}
		=&\frac{1}{Lh}\underbrace{\left(\mathbb{E}[\psi(\Xb_{t})] - \psi(\Xb_0)\right)}_{C_1}
		- \frac{1}{L}\sum_t \mathbb{E}[\Delta V_t\psi(\Xb_{t-1})] \\
		&- O\left(\frac{h}{Lh} + h^{2}\right) + C_3 h^2~,
	\end{align*}	
	for some $C_3 \geq 0$.
	According to the assumption, the term $C_1$ is bounded. As a result, collecting low order terms, the bias can be expressed as:
	\begin{align*}
		&\left|\mathbb{E}\hat{\phi} - \bar{\phi}\right| \\
		=& \left|\mathbb{E}\left(\frac{1}{L}\sum_t\left(\phi(\Xb_{t}) - \bar{\phi}_{\beta_t}\right)\right) +  \frac{1}{L}\sum_t\bar{\phi}_{\beta_t} - \bar{\phi}\right| \\
		\leq& \left|\mathbb{E}\left(\frac{1}{L}\sum_t\bar{\phi}_{\beta_t} - \bar{\phi}\right)\right| + \left|\mathbb{E}\left(\frac{1}{L}\sum_t\left(\phi(\Xb_{t}) - \bar{\phi}_{\beta_t}\right)\right)\right| \\
		\leq& C\phi(\thetab^*)\left(\frac{1}{L}\sum_{t=1}^L\int_{\thetab \neq \thetab^*}e^{-\beta_t\hat{U}(\thetab)}\mathrm{d}\thetab\right) \\
		&+ \left|\frac{C_1}{Lh} - \frac{\sum_t \mathbb{E}\Delta V_t\psi(\Xb_{t-1})}{L} + C_3 h^2\right| \\
		\leq& C\phi(\thetab^*)\left(\frac{1}{L}\sum_{t=1}^L\int_{\thetab \neq \thetab^*}e^{-\beta_t\hat{U}(\thetab)}\mathrm{d}\thetab\right) + \left|\frac{C_1}{Lh}\right| \\
		&+ \left| \frac{\sum_t \mathbb{E}\Delta V_t\psi(\Xb_{t-1})}{L}\right| + \left|C_3 h^2\right| \\
		\leq& C\phi(\thetab^*)\left(\frac{1}{L}\sum_{t=1}^L\int_{\thetab \neq \thetab^*}e^{-\beta_t\hat{U}(\thetab)}\mathrm{d}\thetab\right) \\
		&+ D\left(\frac{1}{Lh} + \frac{\sum_t \left\|\mathbb{E}\Delta V_t\right\|}{L} + h^2\right)~,
	\end{align*}
	where the last equation follows from the finiteness assumption of $\psi$, $\|\cdot\|$ denotes the operator norm
	and is bounded in the space of $\psi$ due to the assumptions. 
	This completes the proof.
\end{proof}

We will now prove the MSE result .

\begin{proof}[Proof of the MSE bound]

Similar to the proof of Theorem~\ref{theo:bias1}, for our 2nd--order integrator we have:
\begin{align*}
	\mathbb{E}\left(\psi_{\beta_t}(\Xb_{t})\right) &= \left(\mathbb{I} + h(\mathcal{L}_{\beta_t} + \Delta V_t)\right) \psi_{\beta_{t-1}}(\Xb_{t-1}) \\
	&+ \frac{h^2}{2}\tilde{\mathcal{L}}_t^2 \psi_{\beta_{t-1}}(\Xb_{t-1}) + O(h^{3})~.
\end{align*}
Sum over $t$ from 1 to $L+1$ and simplify, we have:
\begin{align*}
	\sum_{t=1}^L\mathbb{E}&\left(\psi_{\beta_t}(\Xb_{t})\right) = \sum_{t=1}^L \psi_{\beta_{t-1}}(\Xb_{t-1}) \\
	&+ h\sum_{t=1}^L \mathcal{L}_{\beta_t}\psi_{\beta_{t-1}}(\Xb_{t-1}) + h\sum_{t=1}^L \Delta V_t\psi_{\beta_{t-1}}(\Xb_{t-1}) \\
	&+ \frac{h^2}{2}\sum_{t=1}^L\tilde{\mathcal{L}}_t^2 \psi_{\beta_{t-1}}(\Xb_{t-1}) + O(L h^{3})~.
\end{align*}

Substitute the Poisson equation \eqref{eq:PoissonEq1} into the above equation, divide both sides by $Lh$ 
and rearrange related terms, we have
\begin{align*}
	&\frac{1}{L}\sum_{t=1}^L\left(\phi(\Xb_{t}) - \bar{\phi}_{\beta_t}\right) = \frac{1}{Lh}\left(\mathbb{E}\psi_{\beta_L}(\Xb_{Lh}) - \psi_{\beta_0}(\Xb_0)\right) \\
	&- \frac{1}{Lh}\sum_{t=1}^{L}\left(\mathbb{E}\psi_{\beta_{t-1}}(\Xb_{t-1}) - \psi_{\beta_{t-1}}(\Xb_{t-1})\right) \\
	&- \frac{1}{L}\sum_{t=1}^L \Delta V_t\psi_{\beta_{t-1}}(\Xb_{t-1})
	- \frac{h}{2L}\sum_{t=1}^L\tilde{\mathcal{L}}_t^2 \psi_{\beta_{t-1}}(\Xb_{t-1}) + O(h^2)
\end{align*}

Taking the square of both sides, it is then easy to see there exists some positive constant $C$, such that
\begin{align}\label{eq:mse1}
	&\left(\frac{1}{L} \sum_{t=1}^L \left(\phi(\Xb_{t}) - \bar{\phi}_{\beta_t}\right)\right)^2 \\
	\leq& C\left(\underbrace{\frac{\left(\mathbb{E}\psi_{\beta_L}(\Xb_{Lh}) - \psi_{\beta_0}(\Xb_0)\right)^2}{L^2h^2}}_{A_1}\right. \nonumber\\
	&\left.+ \underbrace{\frac{1}{L^2h^2}\sum_{t=1}^L\left(\mathbb{E}\psi_{\beta_{t-1}}(\Xb_{t-1}) - \psi_{\beta_{t-1}}(\Xb_{t-1})\right)^2}_{A_2} \right.\nonumber\\
	&+ \frac{1}{L^2}\sum_{t=1}^L \Delta V_t^2\psi_{\beta_{t-1}}(\Xb_{t-1}) \nonumber\\ 
	&\left.+ \underbrace{\frac{h^{2}}{2L^2}\left(\sum_{t=1}^L\tilde{\mathcal{L}}_t^2 \psi_{\beta_{t-1}}(\Xb_{t-1})\right)^2}_{A_3} + h^{4}\right) \nonumber
\end{align}

$A_1$ is easily bounded by the assumption that $\|\psi\| \leq V^{r_0} < \infty$. $A_2$ is bounded because it can be shown that
$\mathbb{E}\left(\psi_{\beta_t}(\Xb_{t})\right) - \psi_{\beta_t}(\Xb_{t}) \leq C_1 \sqrt{h} + O(h)$ for $C_1 \geq 0$. Intuitively this
is true because the only difference between $\mathbb{E}\left(\psi_{\beta_t}(\Xb_{t})\right)$ and $\psi_{\beta_t}(\Xb_{t})$ lies
in the additional Gaussian noise with variance $h$.  A formal proof is given in \citet{ChenDC:NIPS15}. Furthermore, $A_3$ is bounded by the following arguments:
\begin{align*}
	A_3 &=  \underbrace{\frac{h^{2}}{2L^2}\left(\sum_{t=1}^L\mathbb{E}\left[\tilde{\mathcal{L}}_t^2 \psi_{\beta_{t-1}}(\Xb_{t-1})\right]\right)^2}_{B_1} \\
	&\hspace{-1cm}+ \underbrace{\frac{h^{2}}{2L^2}\mathbb{E}\left(\sum_{t=1}^L\left(\tilde{\mathcal{L}}_t^2 \psi_{\beta_{t-1}}(\Xb_{t-1}) - \mathbb{E}\tilde{\mathcal{L}}_t^2 \psi_{\beta_{t-1}}(\Xb_{t-1})\right)\right)^2}_{B_2} \\
	&\lesssim B_1 + \left(\frac{h^2}{Lh}\sum_{t=1}^L\tilde{\mathcal{L}}_t^2 \psi_{\beta_{t-1}}(\Xb_{t-1})\right)^2 \\
	&+ \left(\frac{h^2}{Lh}\sum_{t=1}^L\left(\mathbb{E}\tilde{\mathcal{L}}_t^2 \psi_{\beta_{t-1}}(\Xb_{t-1})\right)\right)^2 \\
	&\leq O\left(\frac{1}{2L^2} + L^2h^2\right) + \frac{1}{Lh}\left(\frac{h^{2}}{L}\sum_{t=1}^L(\tilde{\mathcal{L}}_t^2 \psi(\Xb_{t-1}))^2\right) \\
	&+ O\left(\frac{1}{L^2h^2} + h^{4}\right) \\
	&= O\left(\frac{1}{Lh} + L^4\right)
\end{align*}

Collecting low order terms we have:
\begin{align}\label{eq:mse2}
	&\mathbb{E}\left(\frac{1}{L} \sum_{t=1}^L \left(\phi(\Xb_{t}) - \bar{\phi}_{\beta_t}\right)\right)^2 \nonumber\\
	=& O\left(\frac{\frac{1}{L}\sum_t\mathbb{E}\left\|\Delta V_t\right\|^2}{L} + \frac{1}{Lh} + h^{4}\right)~.
\end{align}

Finally, we have:
\begin{align*}
	&\mathbb{E}\left(\hat{\phi} - \bar{\phi}\right)^2 < \mathbb{E}\left(\frac{1}{L}\sum_t\left(\phi(\Xb_{t}) - \bar{\phi}_{\beta_t}\right)\right)^2 \\
	&+ \mathbb{E}\left(\frac{1}{L} \sum_{t=1}^L \left(\phi(\Xb_{t}) - \bar{\phi}_{\beta_t}\right)\right)^2 \\
	\leq& C\phi(\thetab^*)^2\left(\frac{1}{L}\sum_{t=1}^L\int_{\thetab \neq \thetab^*}e^{-\beta_t\hat{U}(\thetab)}\mathrm{d}\thetab\right)^2 \\
	&+ O\left(\frac{\frac{1}{L}\sum_t\mathbb{E}\left\|\Delta V_t\right\|^2}{L} + \frac{1}{Lh} + h^{4}\right) \\
	\leq& C\phi(\thetab^*)^2\left(\frac{1}{L}\sum_{t=1}^L\int_{\thetab \neq \thetab^*}e^{-\beta_t\hat{U}(\thetab)}\mathrm{d}\thetab\right)^2 \\
	&+ D \left(\frac{\frac{1}{L}\sum_t\mathbb{E}\left\|\Delta V_t\right\|^2}{L} + \frac{1}{Lh} + h^{4}\right)~.
\end{align*}

\end{proof}

\section{Proof of Corollary~\ref{coro:refine_con}}

\begin{proof}
The {\em refinement} stage corresponds to $\beta \rightarrow \infty$. We can prove that in this case,
the integration terms in the bias and MSE in Theorem~\ref{theo:bias1}
converge to 0.

To show this, define a sequence of functions $\{g_m\}$ as:
\begin{align}
	g_m \triangleq -\frac{1}{L}\sum_{l=m}^{L+m-1} e^{-\beta_l\hat{U}(\thetab)}~.
\end{align}
it is easy to see the sequence $\{g_m\}$ satisfies $g_{m_1} < g_{m_2}$ for $m_1 < m_2$, and $\lim_{m \rightarrow \infty} g_m = 0$. 
According to the monotone convergence theorem, we have
\begin{align*}
	\lim_{m \rightarrow \infty} \int g_m &\triangleq \lim_{m \rightarrow \infty} \int -\frac{1}{L}\sum_{l=m}^{L+m-1} e^{-\beta_l\hat{U}(\thetab)} \mathrm{d} \thetab \\
	&= \int \lim_{m \rightarrow \infty} g_m = 0~.
\end{align*}

As a result, the integration terms in the bounds for the bias and MSE vanish, leaving only the terms
stated in Corollary~\ref{coro:refine_con}. This completes the proof.
\end{proof}

\section{Reformulation of the Santa Algorithm}

In this section we give a version of the Santa algorithm
that matches better than our actual implementation,
shown in Algorithm~\ref{alg:sahmc1}--\ref{alg:refine_e}.

\begin{algorithm}[h]
\SetKwInOut{Input}{Input}
\caption{Santa}
\Input{$\eta_t$ (learning rate), $\sigma$, $\lambda$, $burnin$,
$\beta = \{\beta_1, \beta_2, \cdots\}\rightarrow \infty$,  $\{\zetab_t \in \RR^p\}\sim \mathcal{N}({\bf 0},\textbf{I}_p)$.}
Initialize $\thetab_{0}$, $\ub_{0} = \sqrt{\eta}\times\mathcal{N}(0, I)$, $\alphab_0 = \sqrt{\eta}C$, $\vb_0 = 0$ \;
\For {$t = 1, 2, \ldots $} {
Evaluate $\tilde{\fb}_t = \nabla_{\thetab} \tilde{U}_t(\thetab_{t-1})$ on the $t$-th minibatch \;
$\vb_t = \sigma \vb_{t-1} + \frac{1 - \sigma}{N^2}\tilde{\fb}_t\odot \tilde{\fb}_t$ \;
$\gb_t = 1\oslash\sqrt{\lambda + \sqrt{\vb_t}}$ \;
\uIf{$t < burnin$}{
\tcc{{\em exploration}}
$(\thetab_t, \ub_t, \alphab_t) = \mbox{Exploration\_S}(\thetab_{t-1}, \ub_{t-1}, \alphab_{t-1})$ 
\vspace{-0.1cm}\begin{center}or\end{center}\vspace{-0.2cm}\
$(\thetab_t, \ub_t, \alphab_t) = \mbox{Exploration\_E}(\thetab_{t-1}, \ub_{t-1}, \alphab_{t-1})$
}
\Else{
\tcc{{\em refinement}}
$(\thetab_t, \ub_t, \alphab_t) = \mbox{Refinement\_S}(\thetab_{t-1}, \ub_{t-1}, \alphab_{t-1})$
\vspace{-0.2cm}\begin{center}or\end{center}\vspace{-0.2cm}\
$(\thetab_t, \ub_t, \alphab_t) = \mbox{Refinement\_E}(\thetab_{t-1}, \ub_{t-1}, \alphab_{t-1})$
}
}
\label{alg:sahmc1}
\end{algorithm}

\begin{algorithm}[!ht]
\caption{$\mbox{Exploration\_S}~(\thetab_{t-1}, \ub_{t-1}, \alphab_{t-1})$}
$\thetab_{t} = \thetab_{t-1} + \gb_t \odot \ub_{t-1} / 2$\;
$\alphab_t = \alphab_{t-1} + \left(\ub_{t-1}\odot \ub_{t-1} - \eta/\beta_t\right) / 2$\;
$\ub_{t} = \exp\left(-\alphab_{t}/2\right)\odot \ub_{t-1}$\;
$\ub_{t} = \ub_{t} - \gb_t\odot \tilde{\fb}_t \eta + \sqrt{2\gb_{t-1}\eta^{3/2}/\beta_t}\odot \zetab_t$\;
$\ub_{t} = \exp\left(-\alphab_{t}/2\right)\odot \ub_{t}$\;
$\alphab_t = \alphab_{t} + \left(\ub_t\odot \ub_t - \eta/\beta_t\right) / 2$\;
$\thetab_{t} = \thetab_{t} + \gb_t\odot \ub_{t} / 2$\;
Return $(\thetab_t, \ub_t, \alphab_t)$
\label{alg:explore_s}
\end{algorithm}

\begin{algorithm}[!ht]
\caption{$\mbox{Refinement\_S}~(\thetab_{t-1}, \ub_{t-1}, \alphab_{t-1})$}
$\alphab_{t} = \alphab_{t-1}$\;
$\thetab_{t} = \thetab_{t-1} + \gb_t \odot \ub_{t-1} / 2$\;
$\ub_{t} = \exp\left(-\alphab_{t}/2\right)\odot \ub_{t-1}$\;
$\ub_{t} = \ub_{t} - \gb_t\odot \tilde{\fb}_t \eta$\;
$\ub_{t} = \exp\left(-\alphab_{t}/2\right)\odot \ub_{t}$\;
$\thetab_{t} = \thetab_{t} + \gb_t\odot \ub_{t} / 2$\;
Return $(\thetab_t, \ub_t, \alphab_t)$
\label{alg:refine_s}
\end{algorithm}

\begin{algorithm}[h]
\caption{$\mbox{Exploration\_E}~(\thetab_{t-1}, \ub_{t-1}, \alphab_{t-1})$}
$\alphab_{t} = \alphab_{t-1} + \left(\ub_{t-1}\odot \ub_{t-1} - \eta/\beta_t\right)$\;
$\ub_{t} = \left(1 - \alphab_{t}\right)\odot \ub_{t-1} - \eta \gb_t\odot \tilde{\fb}_t + \sqrt{2\gb_{t-1}\eta^{3/2}/\beta_t}\odot \zetab_t$\;
$\thetab_{t} = \thetab_{t} + \gb_t \odot \ub_{t}$\; 
Return $(\thetab_t, \ub_t, \alphab_t)$
\label{alg:explore_e}
\end{algorithm}

\begin{algorithm}[!ht]
\caption{$\mbox{Refinement\_E}~(\thetab_{t-1}, \ub_{t-1}, \alphab_{t-1})$}
$\alphab_{t} = \alphab_{t-1}$\;
$\ub_{t} = \left(1 - \alphab_{t}\right)\odot \ub_{t-1} - \eta \gb_t\odot \tilde{\fb}_t$\;
$\thetab_{t} = \thetab_{t} + \gb_t \odot \ub_{t}$\;
Return $(\thetab_t, \ub_t, \alphab_t)$
\label{alg:refine_e}
\end{algorithm}
\vspace{-2mm}

\vspace{-5mm}
\section{Relationship of \textit{refinement} Santa to Adam}
\label{sec:relate_adam}
\vspace{-3mm}
In the Adam algorithm (see Algorithm 1 of \citet{kingma2014adam}), the key steps are:
\begin{align*}
 \tilde{\fb}_t &\triangleq \nabla_{\thetab} \tilde{U}(\thetab_{t-1})\\
\vb_t &= \sigma \vb_{t-1} + (1 - \sigma)\tilde{\fb}_t\odot \tilde{\fb}_t\\
{\gb}_t &= 1\oslash\sqrt{\lambda + \sqrt{\vb_t}}\\
\tilde{\u}_t &= ({\bf1}-\bb_1)\odot {\tilde{\u}}_{t-1} + \bb_1\odot {\tilde{\fb}}_{t}\\
\thetab_t&=\thetab_t+\eta (\gb_t\odot \gb_t) \odot \tilde{\u}_t
\end{align*}
Here, we maintain the square root form of $\gb_t$, so the square is equivalent to the preconditioner used in Adam.  As well, in Adam, the vector $\bb_1$ is set to the same constant between 0 and 1 for all entries.  An equivalent formulation of this is:
\begin{align*}
 \tilde{\fb}_t &\triangleq \nabla_{\thetab} \tilde{U}(\thetab_{t-1})\\
\vb_t &= \sigma \vb_{t-1} + (1 - \sigma)\tilde{\fb}_t\odot \tilde{\fb}_t\\
{\gb}_t &= 1\oslash\sqrt{\lambda + \sqrt{\vb_t}}\\
{\u}_t &= ({\bf1}-\bb_1)\odot {{\u}}_{t-1} -\eta( \gb_t\odot \bb_1\odot {\tilde{\fb}}_{t})\\
\thetab_t&=\thetab_t- \gb_t \odot {\u}_t
\end{align*}
The only differences between these steps and the Euler integrator we present in our Algorithm \ref{alg:sahmc} are that our $\bb_1$ has a separate constant for each entry, and  the second term in $\u$ does not include the $\bb_1$ in our formulation.  If we modify our algorithm to multiply the gradient by $\bb_1$, then our algorithm, under the same assumptions as Adam, will have a similar regret bound of $O(\sqrt{T})$ for a convex problem. 

Because the focus of this paper is not on the regret bound, we only briefly discuss the changes in the theory.  We note that Lemma 10.4 from \citet{kingma2014adam} will hold with element-wise $\bb_1$.  
\begin{lemma}
Let $\gamma_i\triangleq \frac{b_{1,i}^2}{\sqrt{\sigma}}$. For $b_{1,i},\sigma\in [0,1)$ that satisfy $\frac{\beta_1^2}{\sqrt{\beta_2}} <1$ and bounded $\tilde{f}_t$, $||\tilde{f}_t||_2\leq G$, $||\tilde{f}_t||_\infty\leq G_\infty$, the following inequality holds
\begin{align*}
\sum_{t=1}^T \frac{u_i^2}{\sqrt{t g_i^2}}\leq\frac{2}{1-\gamma_i}||\tilde{f}_{1:T,i}||_2
\end{align*}
which contains an element-dependent $\gamma_i$ compared to Adam.
\end{lemma}
Theorem 10.5 of \citet{kingma2014adam} will hold with the same modifications and assumptions for a $\bb$ with distinct entries; the proof in \citet{kingma2014adam} is already element-wise, so it suffices to replace their global parameter $\gamma$ with distinct $\gamma_i\triangleq \frac{b_{1,i}^2}{\sqrt{\sigma}}$.  This will give a regret of $O(\sqrt{T})$, the same as Adam.


%

\vspace{-0mm}
\section{Additional Results}\label{supp:additional_results}
\vspace{-0mm}
\begin{figure}[h]
	\centering
	\includegraphics[width=0.8\linewidth]{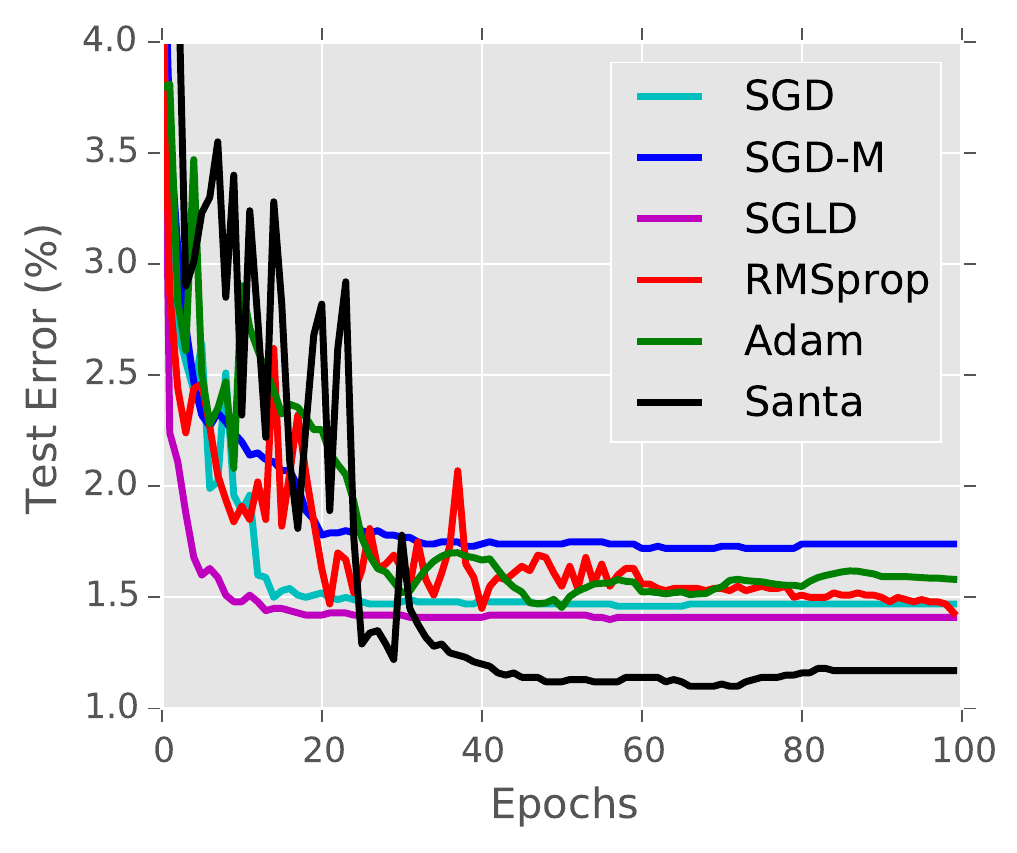}
	\vspace{-2mm}
	\caption{MNIST using FNN with size of 800.}
	\label{fig:fnn_800}
\end{figure}
Learning curves of different algorithms on MNIST using FNN with size of 800 are plotted in Figure~\ref{fig:fnn_800}. Learning curves of different algorithms on four polyphonic music datasets using RNN are shown in Figure~\ref{fig:rnn_add}.

We additionally test Santa on the ImageNet dataset. We use the GoogleNet architecture, which is a 22 layer deep model.
We use the default setting defined in the Caffe package\footnote{$\href{https://github.com/cchangyou/Santa/tree/master/caffe/models/bvlc_googlenet}{https://github.com/cchangyou/Santa/tree/master/caffe/models/bvlc\_googlenet}$}. We were not able to make other stochastic optimization algorithms
except SGD with momentum and the proposed Santa work on this dataset. Figure~\ref{fig:imagenet} shows the comparison 
on this dataset. We did not tune the parameter setting, note the default setting is favourable by SGD with momentum. 
Nevertheless, Santa still significantly outperforms SGD with momentum in term of convergence speed.

\begin{figure}[h]
	\centering
	\includegraphics[width=\linewidth]{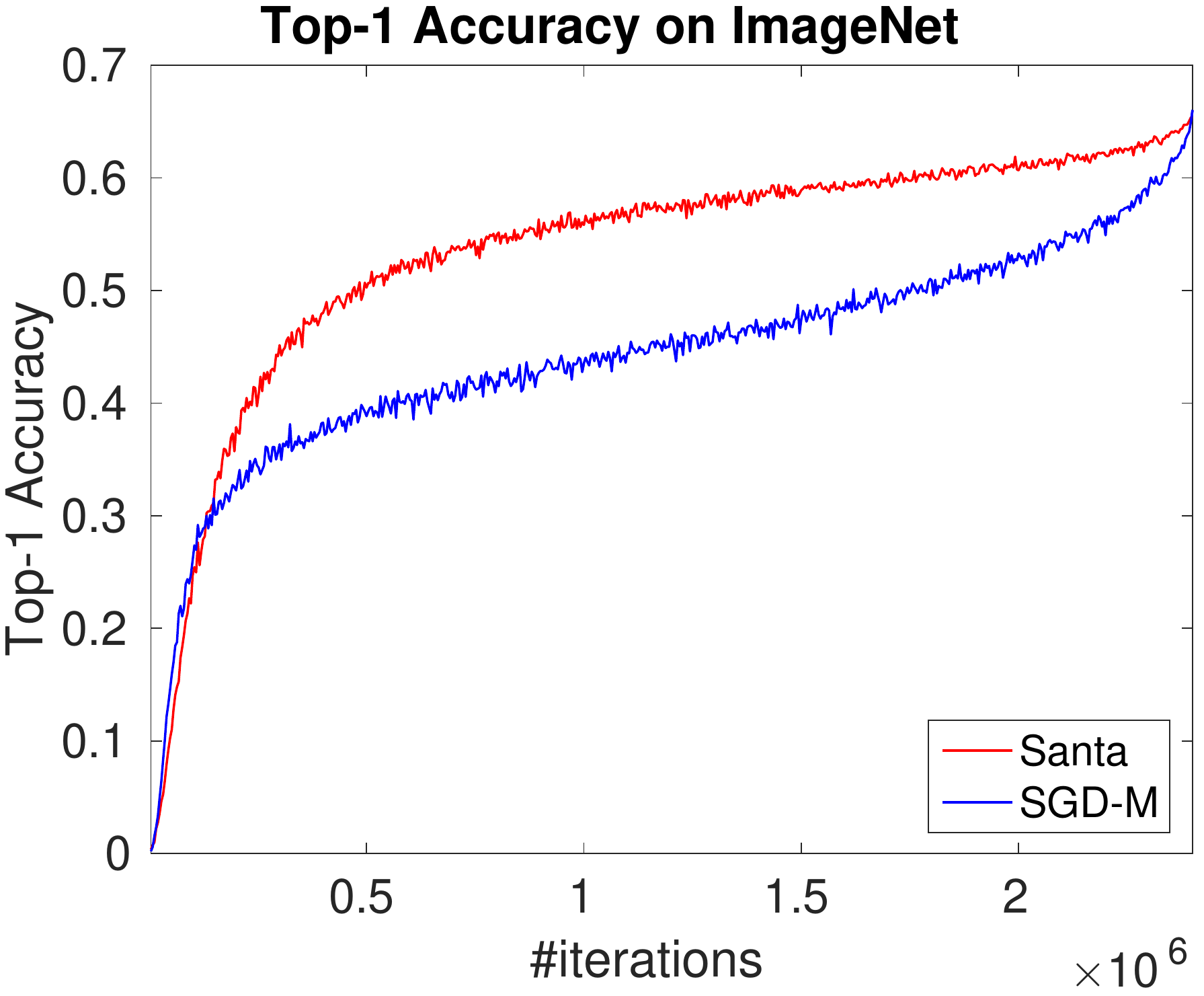}
	\includegraphics[width=\linewidth]{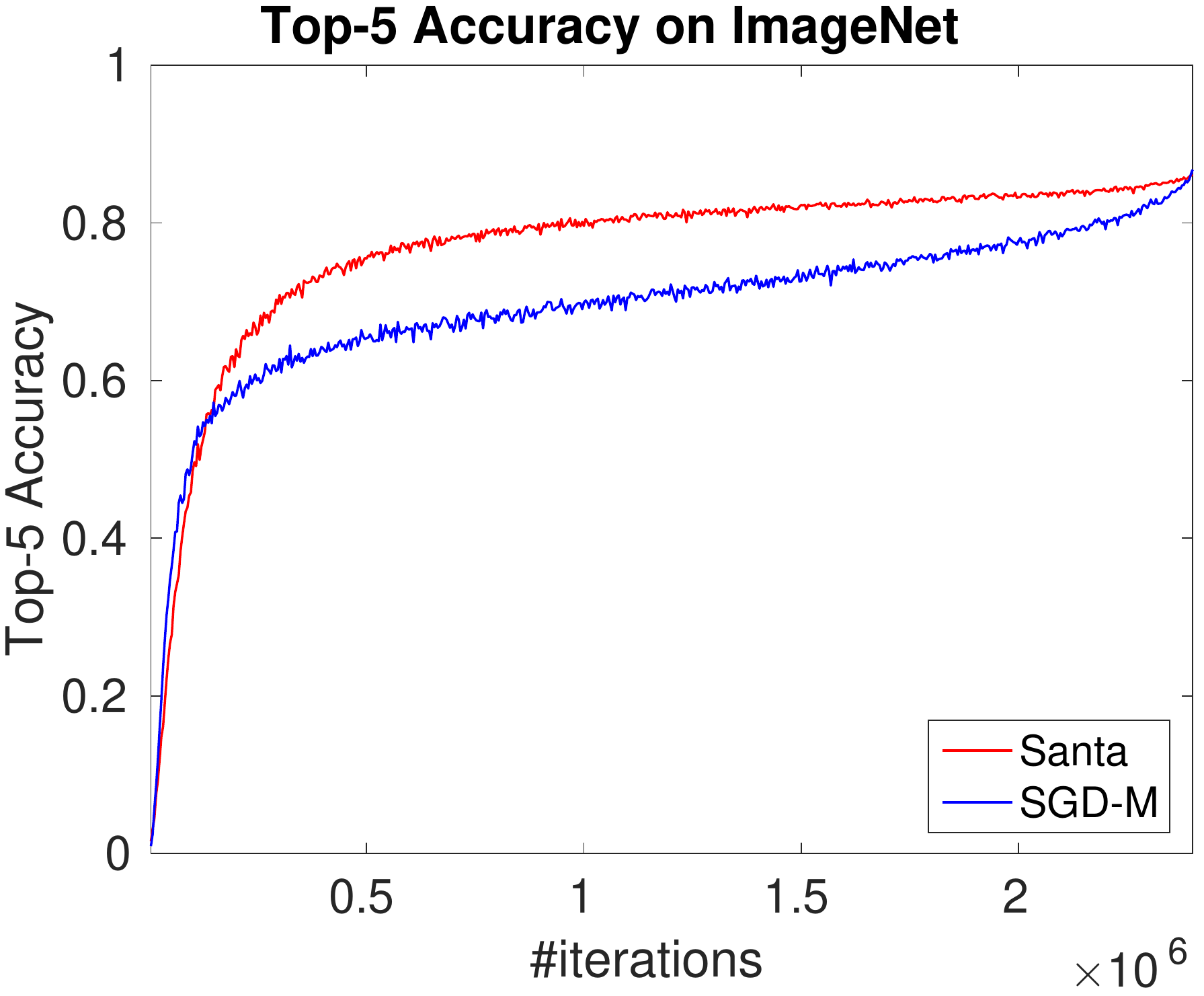}
	\caption{Santa vs. SGD with momentum on ImageNet. We used ImageNet11 for training.}
	\label{fig:imagenet}
\end{figure}

\begin{figure*}[h]
	\centering
    \includegraphics[width=0.32\linewidth]{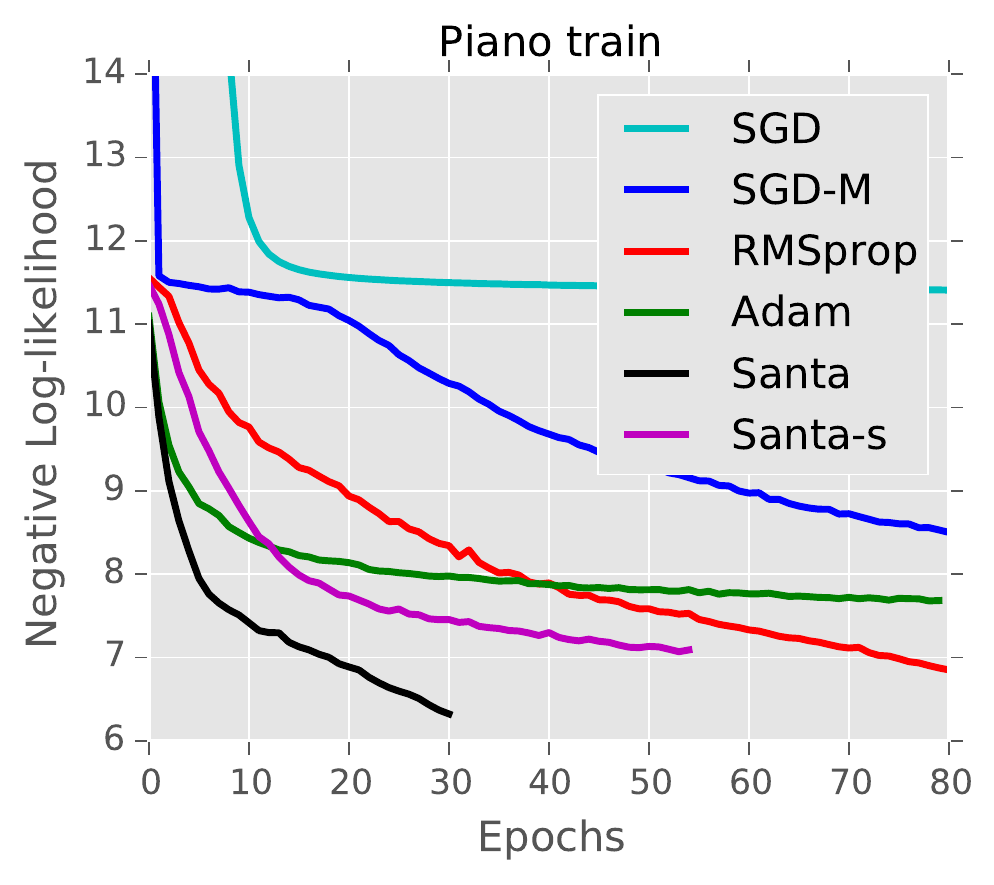}
    \includegraphics[width=0.32\linewidth]{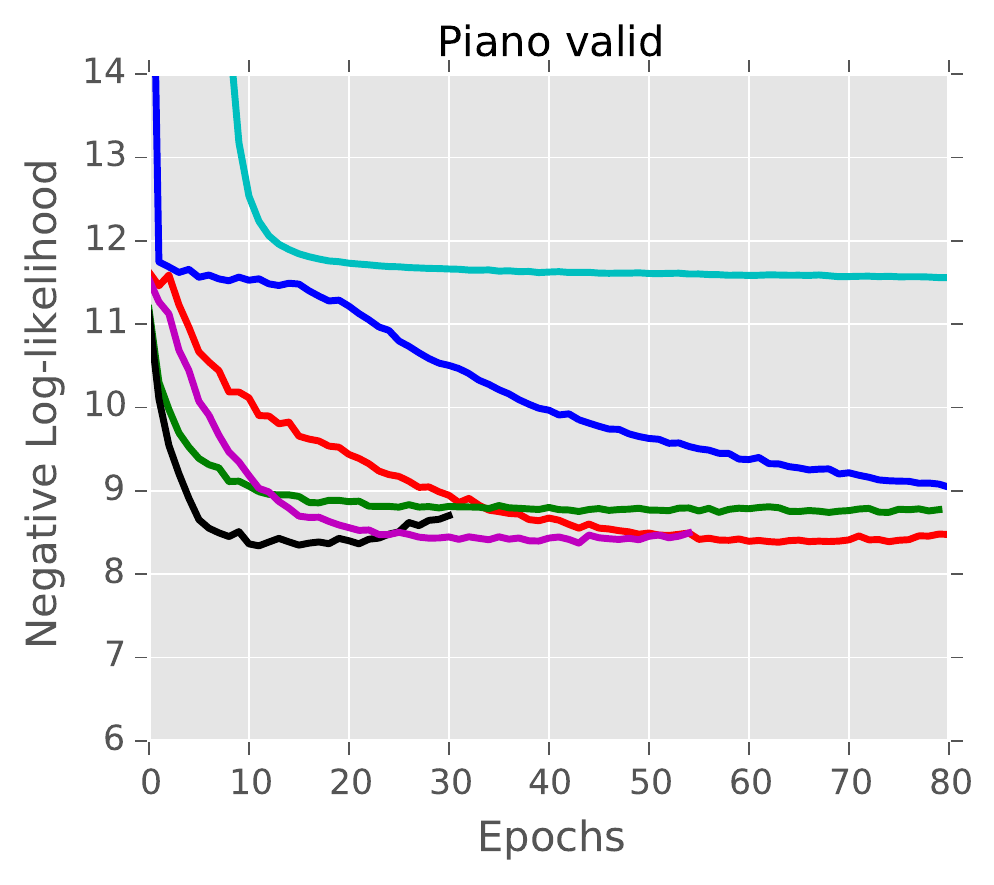}
    \includegraphics[width=0.32\linewidth]{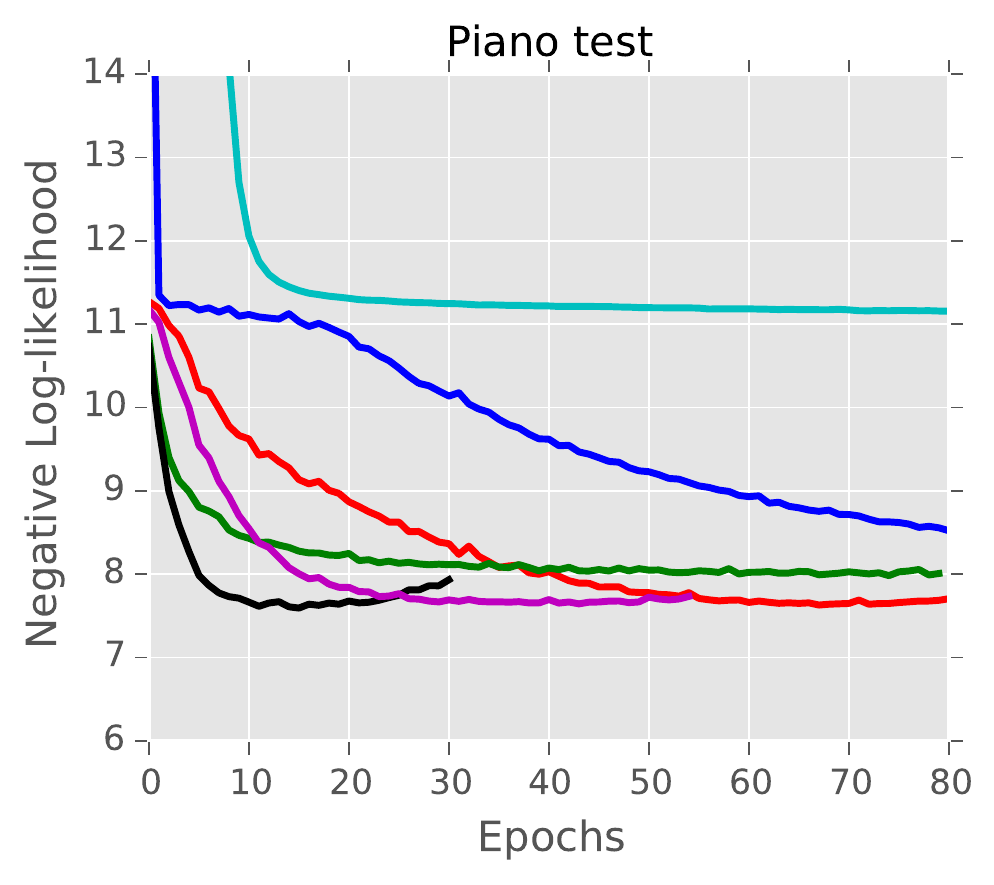} \\
    \includegraphics[width=0.32\linewidth]{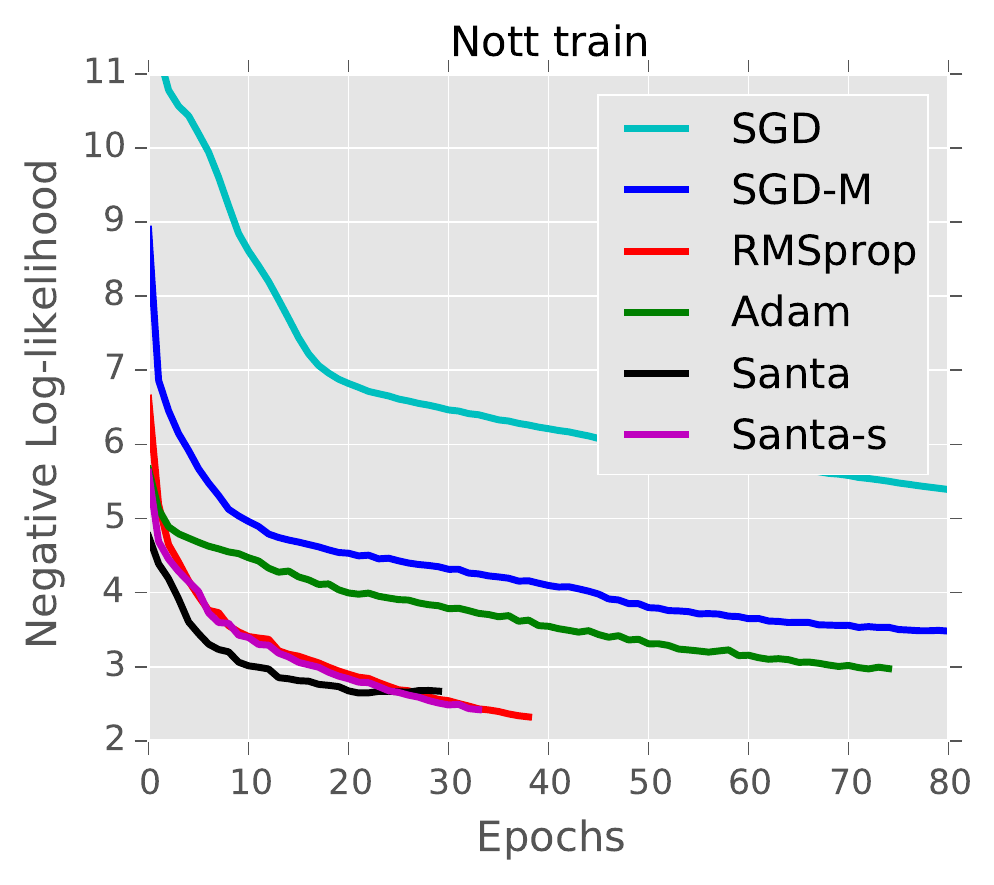}
        \includegraphics[width=0.32\linewidth]{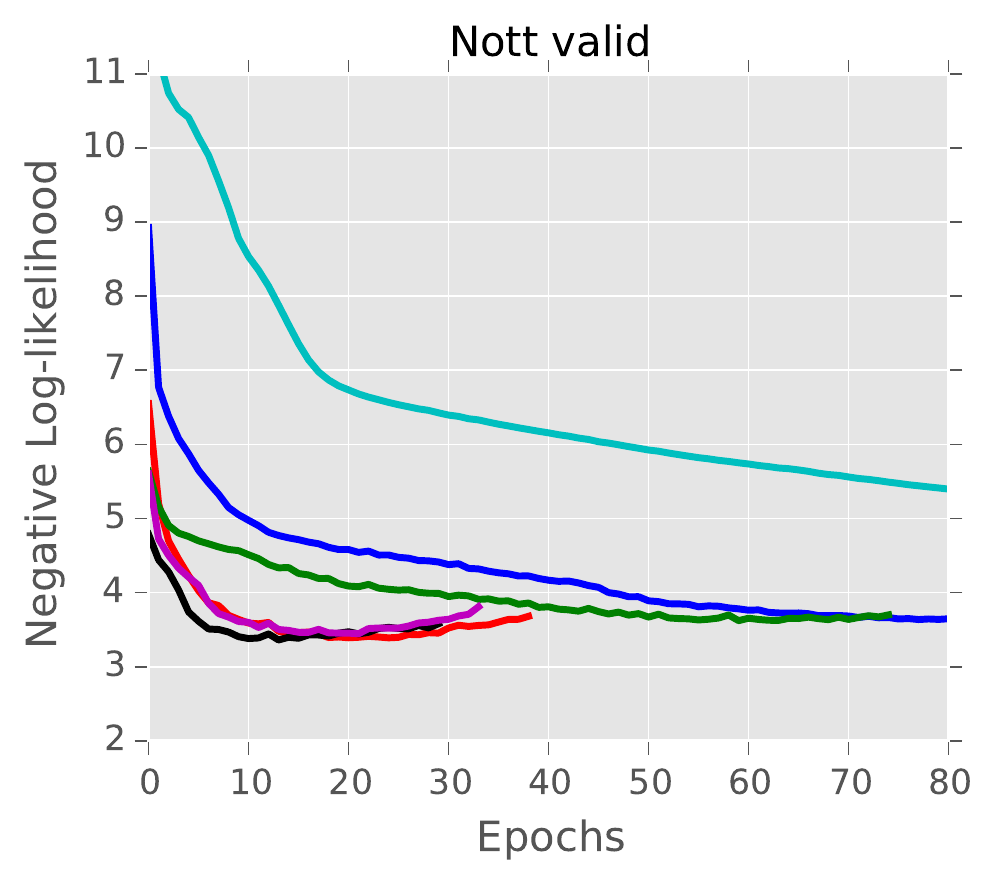}
        \includegraphics[width=0.32\linewidth]{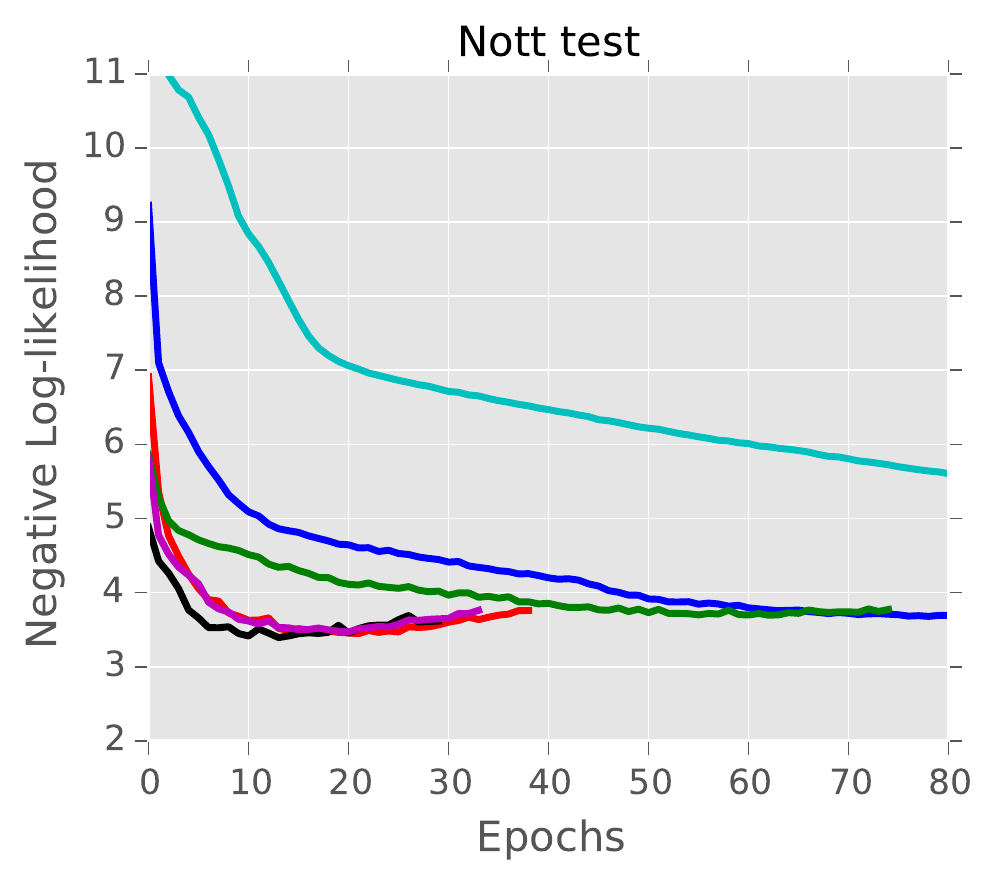} \\
        \includegraphics[width=0.32\linewidth]{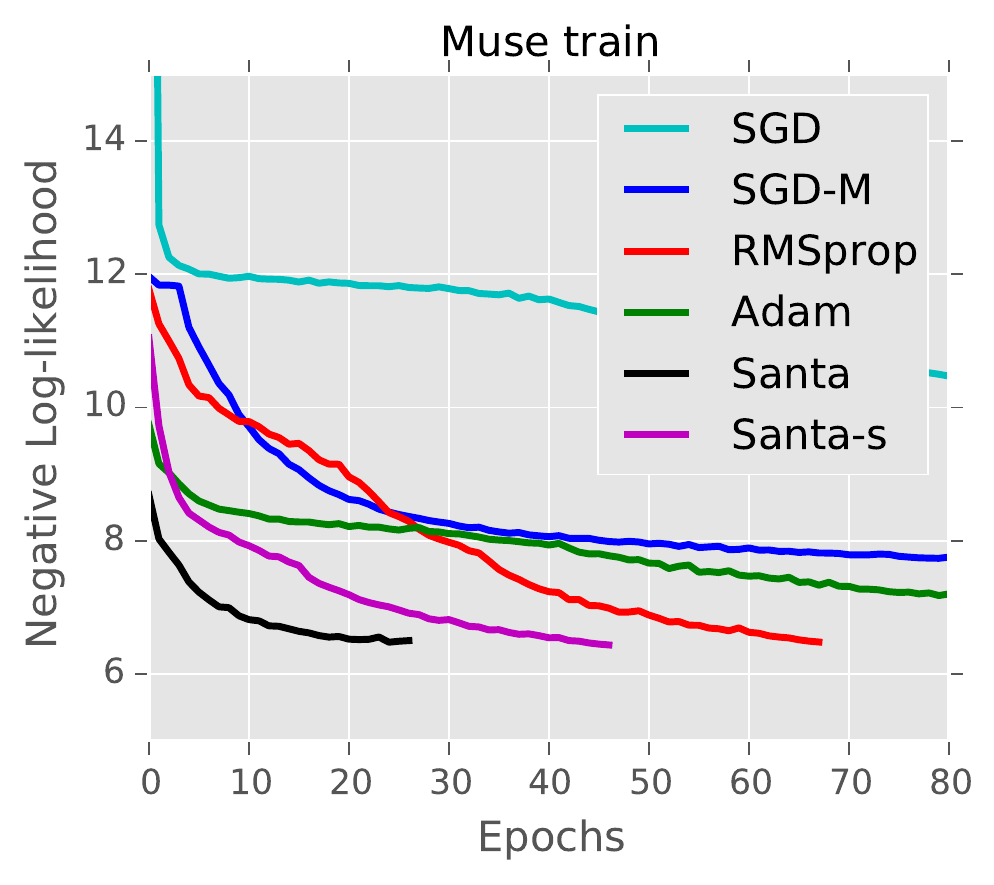}
            \includegraphics[width=0.32\linewidth]{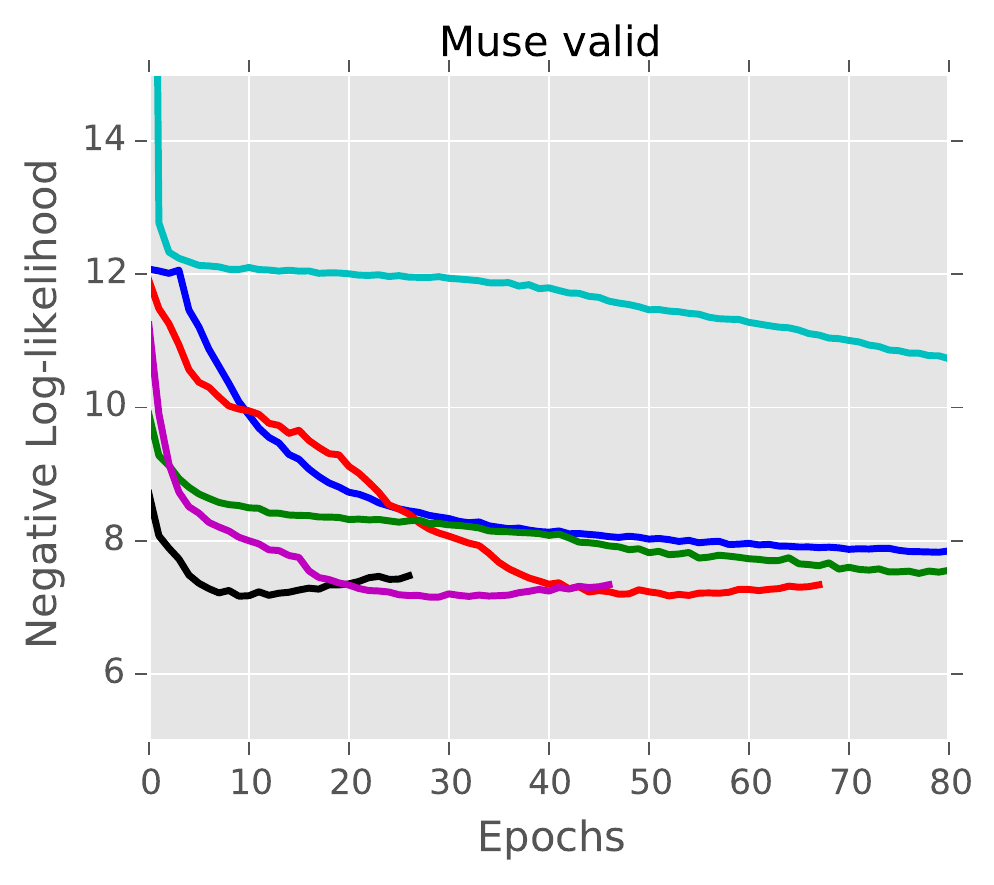}
            \includegraphics[width=0.32\linewidth]{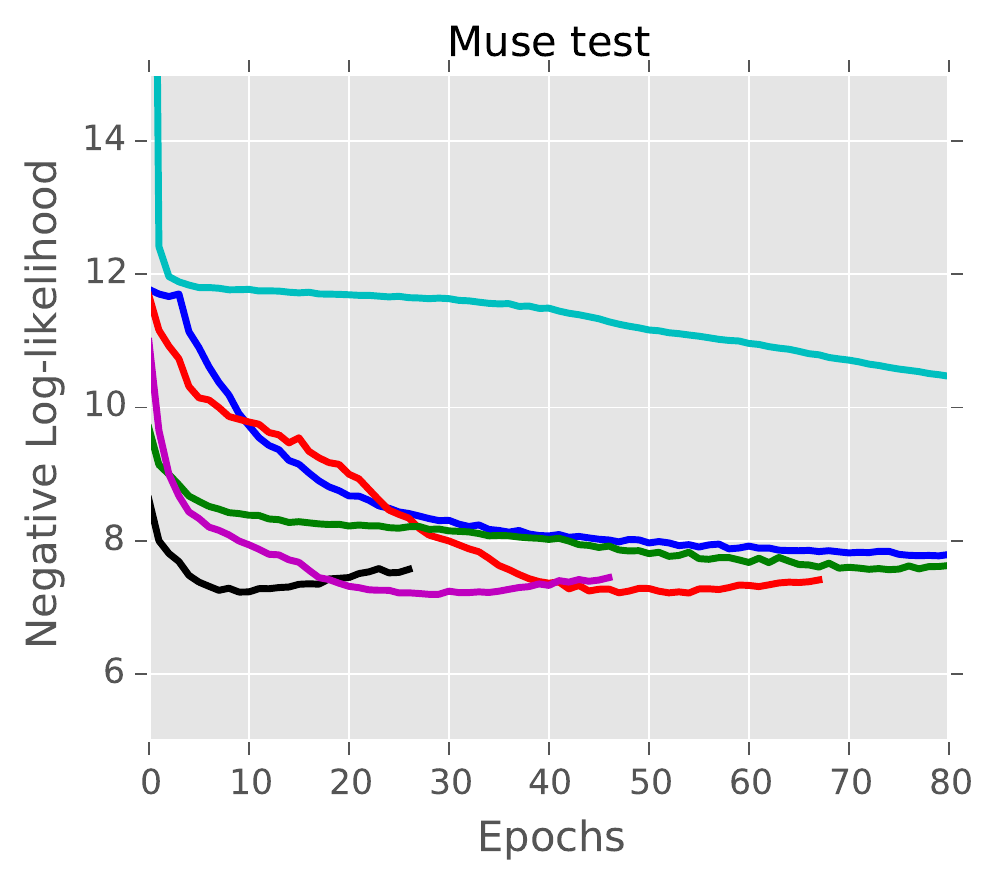} \\
             \includegraphics[width=0.32\linewidth]{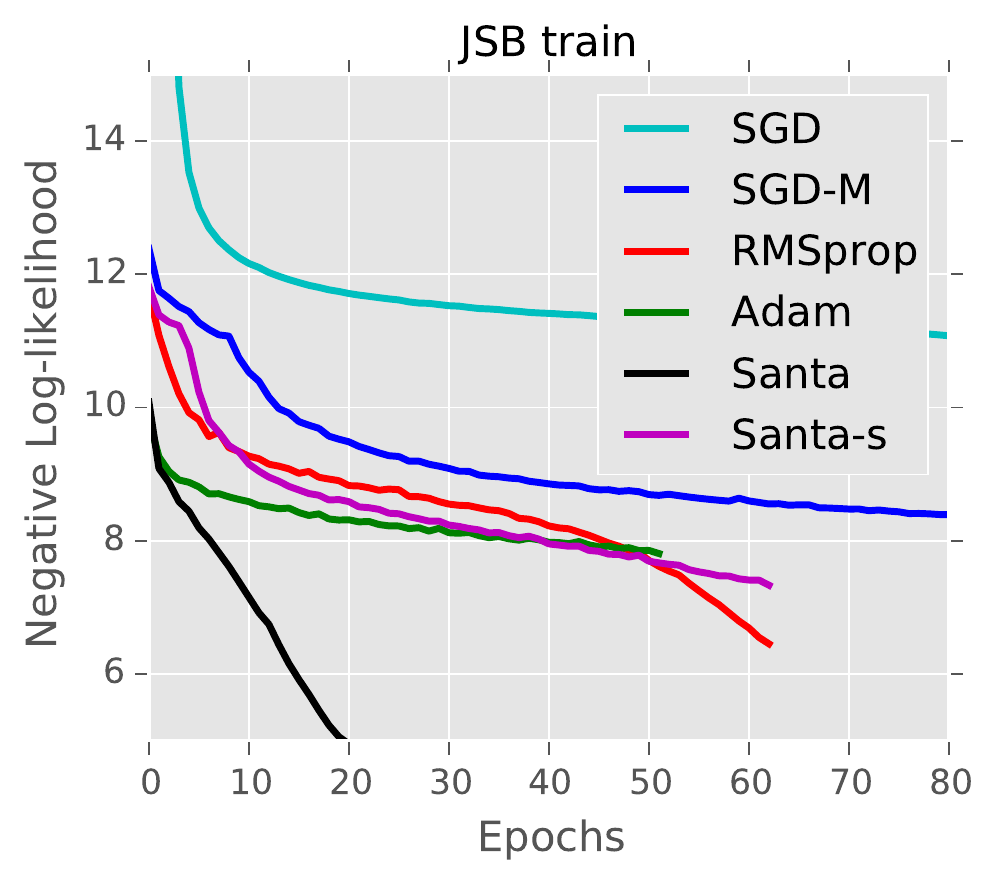}
                \includegraphics[width=0.32\linewidth]{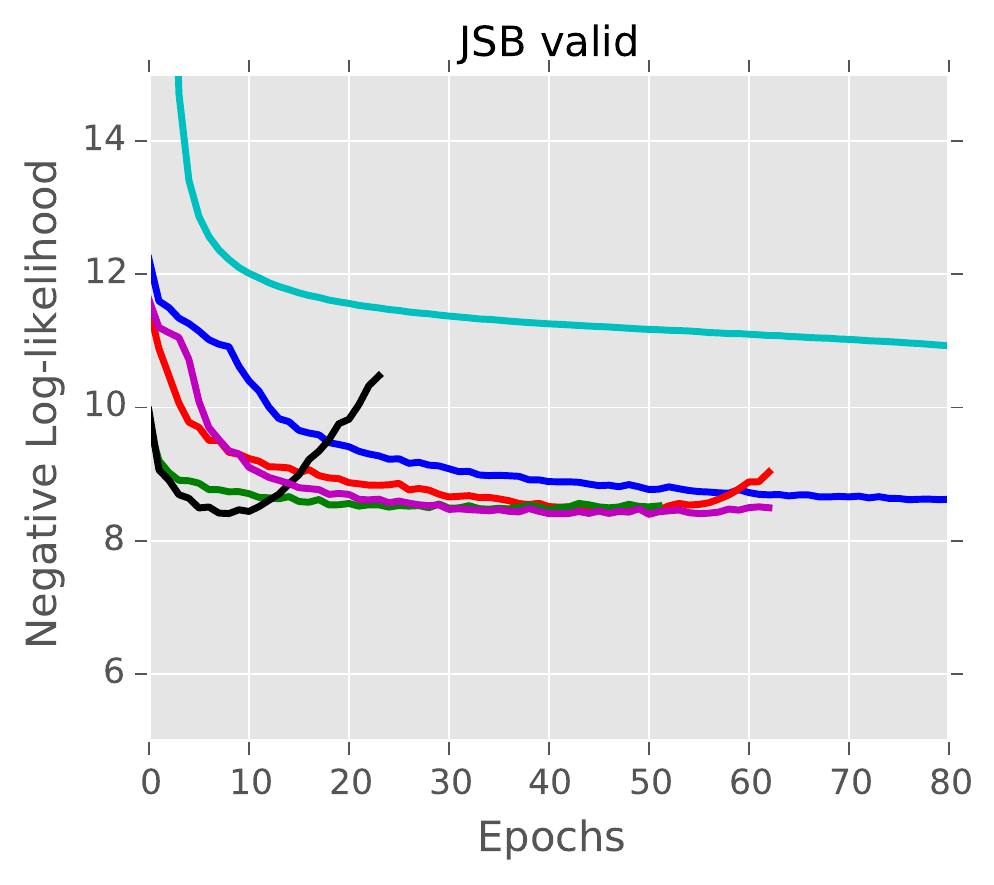}
                \includegraphics[width=0.32\linewidth]{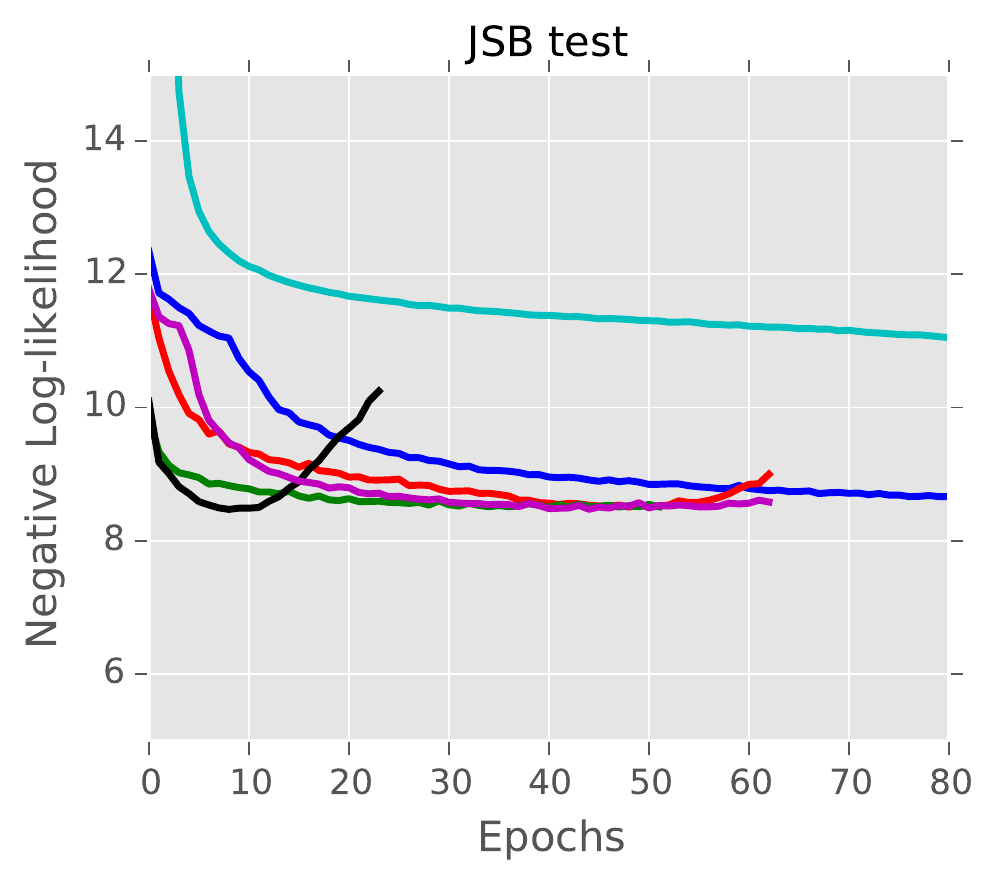}
	\caption{Learning curves of different algorithms on four polyphonic music datasets using RNN.}
	\label{fig:rnn_add}
\end{figure*}